\documentclass{article}

\usepackage{iclr2027_conference,times}
\usepackage{microtype}
\usepackage{graphicx}
\usepackage{subcaption}
\usepackage{booktabs}
\usepackage{float}
\usepackage{hyperref}
\usepackage{url}
\usepackage{amsmath}
\usepackage{amssymb}
\usepackage{mathtools}
\usepackage{amsthm}
\usepackage[capitalize,noabbrev]{cleveref}
\theoremstyle{plain}
\newtheorem{theorem}{Theorem}[section]

\newtheorem{corollary}[theorem]{Corollary}
\theoremstyle{definition}

\theoremstyle{remark}

\newcommand{\lstate}{\textsc{L-State}}
\newcommand{\caponly}{\textsc{Capability}}
\newcommand{\ophead}{\textsc{Operator}}

\newcommand{\sourcegate}{\textsc{Action-Wise Pulse Selector}}
\newcommand{\R}{\mathbb{R}}

\title{\mbox{Target-Independent Micro-Interventions}\\
\mbox{for Predicting Training Response Across}\\
\mbox{Language-Model Families}}
\author{Zhongxuan Liu \qquad Sicheng Zhou \qquad Hongzhi Wang\thanks{Corresponding author.}\\
Faculty of Computing, Harbin Institute of Technology\\
\texttt{lazrix@163.com} \quad \texttt{ylnfq\_2021@qq.com}\\
\texttt{wangzh@hit.edu.cn}}
\iclrfinalcopy

\begin{document}

\maketitle
\lhead{Preprint}

\begin{abstract}
Benchmark scores describe what a checkpoint can do now, but they do not
determine how it will respond to the next training episode. We measure this
missing state by branching four short, standardized, target-independent
micro-interventions from the same checkpoint and recording their effects in a
common capability space. Together with current capability, these responses
form \lstate{}; its pulse block supports a flexible direct readout and a
structure-preserving operator readout. Under smooth local dynamics, the
operator construction admits an end-to-end cross-family bound with explicit
source- and target-family coordinate heterogeneity. In three-family
leave-one-family-out development, both pulse readouts reduce source-standardized
MSE by 39.4\% relative to capability alone, while separating the best response
and direction estimates. On sealed GLM-4-9B, the direct and operator readouts
reduce MSE by 71.8\% and 78.3\%, respectively, and the operator readout raises
sign balanced accuracy from 0.366 to 0.754. On sealed Granite-3.1-8B, the direct
readout reaches RMSE 0.544 and a development-fitted action-wise selector reaches
0.554, compared with 1.172 for capability alone. A five-family audit finds that
the operator coordinate varies by action and family, and that modeling these
deviations improves retrospective held-trajectory prediction. Target-independent
interventions therefore expose training-response information that current
capability misses, with direct and structured readouts covering complementary
transfer regimes.
\end{abstract}

\section{Introduction}

Suppose two language-model checkpoints obtain the same scores on mathematics,
code, science, and reading. Are they equally good starting points for another
round of training? Identical current scores leave the answer unresolved. One
checkpoint may learn a new skill quickly, another may forget neighboring
skills, and a third may remain nearly unchanged. Present-day scores record
what a model currently does; controlled micro-training records how that
checkpoint changes under a fixed training contract.

We study whether the missing information can be measured by intervention. From
the same checkpoint, we independently branch four short, standardized, and
target-independent training episodes. We evaluate the capability change caused
by each episode and concatenate those responses with the checkpoint's current
capabilities. This 25-dimensional object is an \lstate{}: five current
capabilities plus four five-dimensional response vectors (Figure
\ref{fig:pipeline}). Figure~\ref{fig:overview} connects this construction to the
complete theory-to-evidence program. Its 20-dimensional pulse block supports
two complementary cross-family readouts. The direct readout learns a
multi-output map from the measured pulse responses to a target-action response.
The operator readout organizes those same measurements as a local response
operator and applies a source-estimated action coordinate. Both readouts live
in a common evaluation space, so families with different parameters, gradients,
and hidden representations can be compared through the same measured responses.

\begin{figure*}[t]
  \centering
  \includegraphics[width=0.98\textwidth]{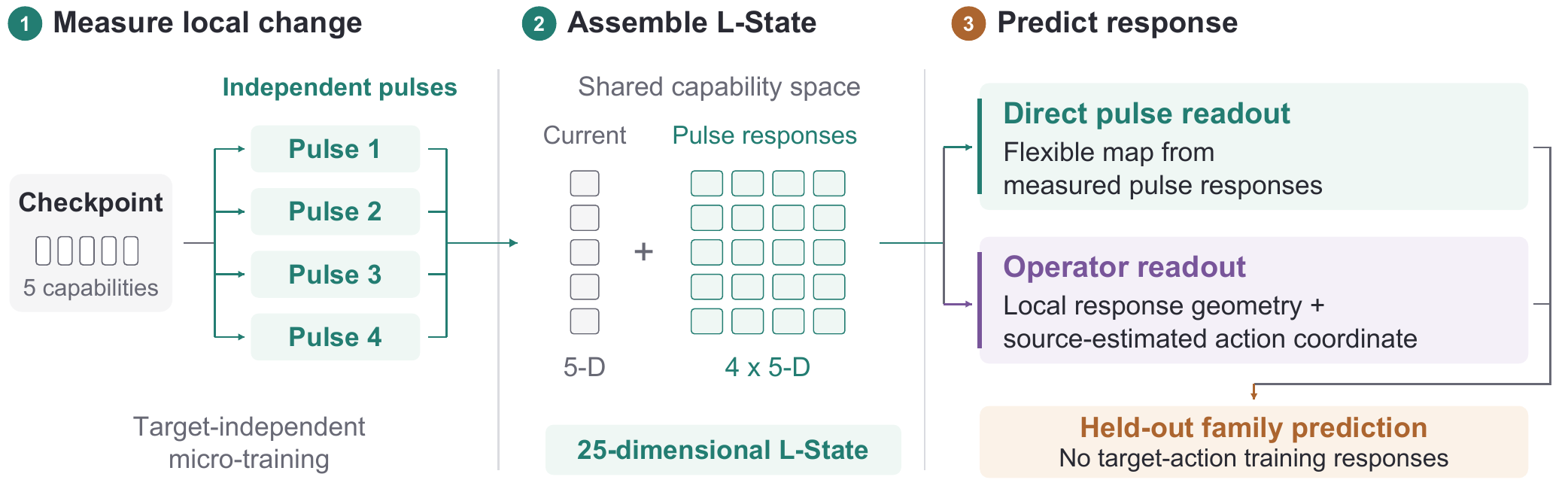}
  \caption{\textbf{Why target-independent interventions reveal training
  response.} Current capability describes a checkpoint before the next update.
  Four standardized pulses branch from that checkpoint and measure local
  changes in a shared capability space. The resulting L-State supports direct
  and operator-factorized prediction for a held-out family without using its
  target-action labels.}
  \label{fig:pipeline}
\end{figure*}

\begin{figure*}[t]
  \centering
  \includegraphics[width=0.98\textwidth]{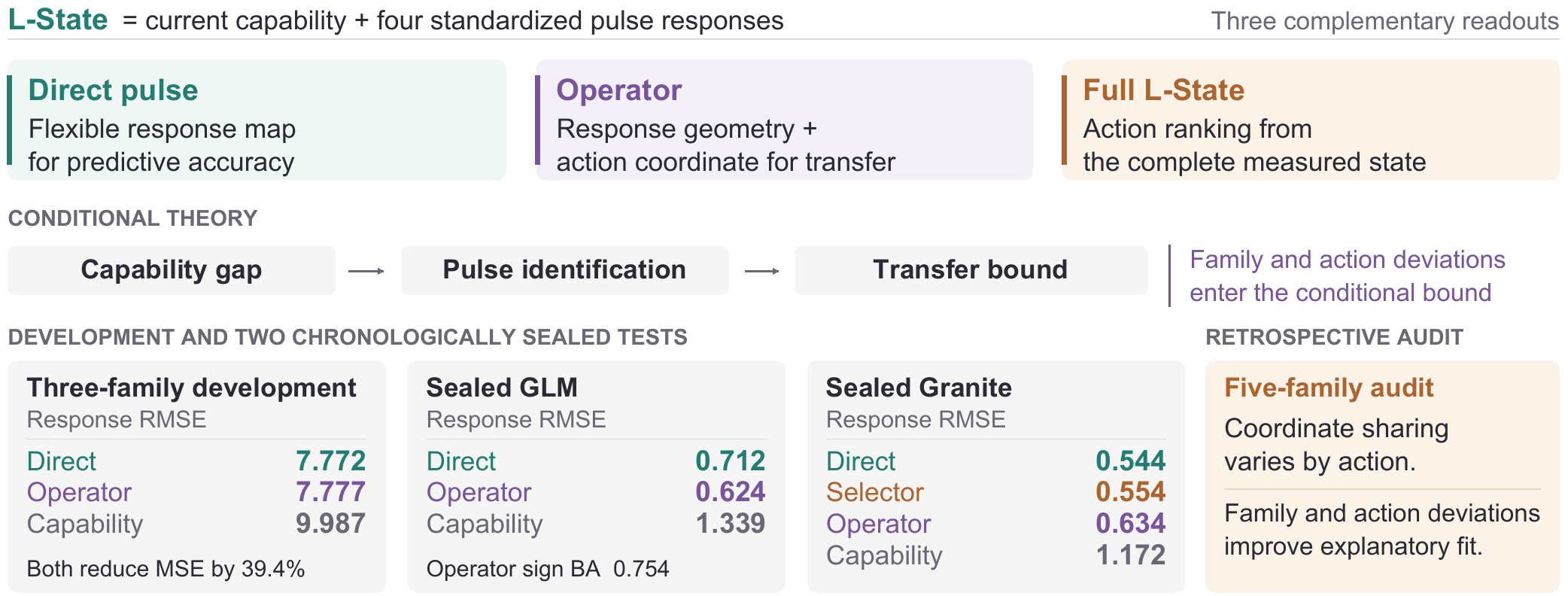}
  \caption{\textbf{Theory and evidence for cross-family training-response
  prediction.} L-State supports direct, operator-factorized, and full-state
  readouts. Conditional theory motivates the operator construction, two sealed
  tests evaluate transfer, and a five-family audit measures action- and
  family-specific coordinate deviations.}
  \label{fig:overview}
\end{figure*}

The resulting prediction problem is deliberately strict. To predict the
response of Qwen, for example, a readout may use target-action responses from
Mistral and OLMo and may measure Qwen's current capability and standard-pulse
responses. Qwen target-action responses remain sealed until every prediction
byte is frozen. This leave-one-family-out (LOFO) design measures transfer of a
behavioral response state across distinct model families.

Our core insight is that target-independent pulse responses form a transferable
training state with two complementary readouts. The direct readout learns a
flexible map from the pulse block; the operator readout factors prediction into
local response geometry and a source-estimated action coordinate. Under the
conditions in Section~3, Theorem \ref{thm:end-to-end} supports the operator and
Eq. \eqref{eq:e2e} separates target identification, source calibration, and
cross-family coordinate bias. Development data determine which structure best
serves each endpoint.

In three-family LOFO development, both pulse readouts reduce continuous-response
MSE by 39.4\% relative to capability alone: the direct readout has the lowest
response RMSE and the operator the highest sign balanced accuracy. We freeze
the operator for response and direction and the full-state readout for action
scoring before opening GLM-4-9B labels. On GLM, direct and operator MSE reductions
reach 71.8\% and 78.3\%, while the operator raises sign balanced accuracy by
38.8 points over capability.

GLM also exposes an operator inversion on science: all 27 states and four
additional realizations preserve the reversed sign, while the source action
coordinate is shrunken and unstable. We therefore freeze an action-wise selector
from held-out development-family wins before opening Granite-3.1-8B. On Granite,
the direct readout reaches RMSE 0.544, the selector 0.554, the operator retains
0.657 coordinate-macro sign balanced accuracy, and the full \lstate{} readout
reduces action regret by 25.5\%.

Our contributions are:
\begin{itemize}
  \item We introduce an intervention-derived learning state with direct and
  structure-preserving operator readouts.
  \item We derive conditional response, sign, regret, and exact-action guarantees
  with explicit source- and target-family coordinate-bias terms.
  \item Across development and two sealed families, pulse readouts improve
  response prediction over capability; a development-fitted selector reaches
  RMSE 0.554 on Granite versus 1.172 for capability.
  \item We test pooled action coordinates across five families, quantify family
  and family--action deviations, and link heterogeneity to readout selection.
\end{itemize}

\section{Related Work}

\paragraph{Predicting adaptation.}
Static diagnostic probes can predict fine-tuning performance in studied NLP
settings \citep{zhu2022probing}; analyses of intermediate checkpoints reveal
training regularities shared across scales \citep{xia2023trajectories}; and
observational scaling laws compress benchmark measurements into a low-dimensional
capability space \citep{ruan2024observational}. Transferability scores similarly
ask which pretrained representation or checkpoint will adapt well
\citep{nguyen2020leep,munn2025bayesian}. Resource-constrained model selection
has also been formulated as extrapolating full fine-tuning performance from
smaller-data runs through a rectified scaling law \citep{lin2024rectified}.
Model-level representation work constructs compact embeddings for correctness,
routing, and benchmark-performance prediction \citep{zhuang2025embedllm}, and
training-free functional fingerprints compare heterogeneous models
\citep{wu2026llmdna}. These methods characterize current functional behavior;
\lstate{} derives its representation from controlled training interventions and
predicts candidate-action response vectors.
Closest in experimental goal, \textsc{TuneAhead} combines dataset descriptors
with a short standardized probe to forecast scalar final scores for runs based
on Qwen2.5-7B-Instruct \citep{luo2026tuneahead}. A concurrent preprint is closer
in mechanism: it models learning through a parameter--optimizer receiving state
and target-specific update geometry in nanoGPT, ResNet, and diffusion
experiments \citep{wang2026unified}. \lstate{} differs in representation and
validation: current capabilities and target-independent pulse responses live in
a common evaluation space, predict a signed response vector over several
candidate actions, and are evaluated by family-held-out development followed by
chronologically sealed tests on GLM-4-9B and Granite-3.1-8B with downstream
action regret.

\paragraph{From training data to behavior.}
Datamodels learn how training-set composition changes predictions
\citep{ilyas2022datamodels}; influence functions and optimizer-aware gradient
methods trace or rank training examples by their target effect
\citep{koh2017influence,xia2024less}. Task arithmetic represents completed
fine-tuning runs as directions in a shared parameter space
\citep{ilharco2023task}. Recent work transports completed task vectors across
heterogeneous-width models by aligning observed internal activations and their
functional effect \citep{rinaldi2026theseus}. \lstate{} takes a complementary
route: it treats the checkpoint as the object being identified, measures
target-independent pulse responses in evaluation space, and predicts future
candidate-action responses. Our question is also conditional in the sense of
\citet{hewitt2021conditional}: do pulse responses provide usable information
beyond current capability?

\paragraph{Intervention and domain transfer.}
Informative controlled inputs are central to active system identification
\citep{wagenmaker2020active}. This analogy motivates standardized pulses, but
our local response model states the training-domain assumptions explicitly.
Following domain-generalization evaluation lessons
\citep{gulrajani2021domainbed}, family identity defines the outer split and all
selection occurs on source families. The prospective stage extends this design
with chronological seals: endpoint rules, response predictions, and
utility-conditioned action choices are serialized before held-out-family
target-action labels enter scoring.

\section{L-State and Its Pulse Readouts}

\subsection{Why Static Capability Is Insufficient}

Each model family $m$ has a state space $\mathcal X_m$. A shared evaluation
contract maps $x\in\mathcal X_m$ to capabilities $c_m(x)\in\R^d$, with larger
coordinates better. For a micro-training action $u$, strength $h$, and training
randomness $\xi$, define the population response
\begin{equation}
 r_m^h(x,u)=\frac{\mathbb E_{\xi}[c_m(U^h_{m,u}(x;\xi))]-c_m(x)}{h}.
 \label{eq:response}
\end{equation}

\begin{theorem}[Capability sufficiency and non-identifiability]
\label{thm:capability}
If $c_m(x)=c_{m'}(x')=c$, write
$e=\|f(c,u)-r_m^h(x,u)\|_2$ and
$e'=\|f(c,u)-r_{m'}^h(x',u)\|_2$. Every capability-only predictor satisfies
\begin{equation}
 \max\{e,e'\}
 \geq \tfrac12\|r_m^h(x,u)-r_{m'}^h(x',u)\|_2.
 \label{eq:cap-lower}
\end{equation}
Moreover, an exact capability-only response map exists on a domain if and only
if, for every $u$, the response is constant on every capability level set.
\end{theorem}

Thus the missing state variable is precisely the within-level-set variation of
training response. Under the geometry developed next, controlled interventions
make that variation observable.

\subsection{Local Intervention Geometry}

The transfer chain uses five explicit protocol conditions: a named pulse gauge,
smooth local dynamics, spanning pulse excitation, positive source coverage, and
bounded measurement and misspecification errors. Whether target actions have
the same coordinate across families is left as a hypothesis and enters the
bounds quantitatively.

In the common named-pulse gauge, fix a reference family set and positive
normalized weights $\{\pi_m\}_m$ and $\{\omega_u\}_u$. Write
\begin{equation}
 a_{m,u}=a_0+\alpha_u+\beta_m+\gamma_{m,u}
 =\bar a_u+\delta_{m,u},
 \label{eq:coordinate-decomposition}
\end{equation}
where $\bar a_u=a_0+\alpha_u$ is pooled, $\beta_m$ is the family deviation, and
$\gamma_{m,u}$ is the family--action interaction. Identifiability uses
\begin{equation}
 \sum_u\omega_u\alpha_u=0,\quad
 \sum_m\pi_m\beta_m=0,\quad
 \sum_m\pi_m\gamma_{m,u}=0\ \forall u,\quad
 \sum_u\omega_u\gamma_{m,u}=0\ \forall m.
 \label{eq:coordinate-centering}
\end{equation}
Thus $\bar a_u=\sum_m\pi_m a_{m,u}$ and
$\delta_{m,u}=a_{m,u}-\bar a_u$. Exact sharing is the nested hypothesis
\begin{equation}
 H_{\rm share}:\quad \beta_m=0\ \text{and}\ \gamma_{m,u}=0
 \quad\text{for all }m,u.
 \label{eq:sharedness-hypothesis}
\end{equation}

Work in a local Euclidean chart around $x$ and write the random training
displacement as $\Delta_{m,u}^h(x,\xi)$. Let the capability Jacobian be
$L_c$-Lipschitz along every segment traversed by this displacement. Suppose the
family-specific coordinate $a_{m,u}\in\R^r$ and a checkpoint-specific update
map $G_m(x)$ obey
\begin{equation}
 \begin{aligned}
 \left\|\frac{\mathbb E\Delta_{m,u}^h}{h}-G_m(x)a_{m,u}\right\|_2
 &\leq\epsilon_g,\\
 \mathbb E\|\Delta_{m,u}^h\|_2^2
 &\leq h^2V^2\|a_{m,u}\|_2^2.
 \end{aligned}
 \label{eq:update-factor}
\end{equation}

\begin{theorem}[Smooth dynamics induce the response factorization]
\label{thm:smooth-factor}
Under Eq. \eqref{eq:update-factor}, the response in Eq. \eqref{eq:response}
has the form
\begin{equation}
 \begin{aligned}
 r_m^h(x,u)&=B_m(x)a_{m,u}+d_m(x,u,h),\\
 B_m(x)&=Jc_m(x)G_m(x),
 \end{aligned}
 \label{eq:local}
\end{equation}
with
\begin{equation}
 \|d_m(x,u,h)\|_2
 \leq \|Jc_m(x)\|_2\epsilon_g
 +\tfrac12L_chV^2\|a_{m,u}\|_2^2.
 \label{eq:local-remainder}
\end{equation}
\end{theorem}

The first term measures departure of the expected update from the
family-specific coordinate factorization; the second is the finite-radius
curvature cost. Training randomness enters through the expected displacement
and its second moment, while repeat averaging controls the empirical estimation
noise below. Cross-family sharing is governed separately by
Eq. \eqref{eq:sharedness-hypothesis}.

\subsection{Identifying the Checkpoint Response Operator}

For a fixed family, let $k$ target-independent pulses $q_1,\ldots,q_k$ have
coordinate matrix $Q_m=[a_{m,q_1},\ldots,a_{m,q_k}]\in\R^{r\times k}$.
Their population and empirical response matrices satisfy
\begin{equation}
 Y_m=B_mQ_m+D_m,\qquad \widehat Y_m=B_mQ_m+E_m,
 \qquad E_m=D_m+Z_m,
 \label{eq:pulse-matrix}
\end{equation}
where $D$ collects local residuals and $Z$ collects training and evaluation
noise.

\begin{theorem}[Pulse identifiability]
\label{thm:pulse-ident}
For a fixed target coordinate $a_\star$, the response $Ba_\star$ is uniquely
determined from $BQ$ for every $B\in\R^{d\times r}$ if and only if
$a_\star\in\operatorname{col}(Q)$. Hence every action in $\R^r$ is identifiable
if and only if $Q$ has full row rank. In that case,
$\widehat B=\widehat YQ^\dagger$ satisfies
\begin{equation}
 \|\widehat B-B\|_{\nu}
 \leq \frac{\|E\|_{\nu}}{\sigma_r(Q)},
 \qquad \nu\in\{2,F\}.
 \label{eq:ident}
\end{equation}
\end{theorem}

\begin{theorem}[Minimax pulse conditioning]
\label{thm:pulse-opt}
For full-row-rank $Q$, the exact worst-case spectral amplification is
\begin{equation}
 \sup_{\|E\|_2\leq1}\|EQ^\dagger\|_2
 =\frac{1}{\sigma_r(Q)}.
\end{equation}
Under $\|Q\|_F^2\leq E_0$ and $k\geq r$, its minimum is
$\sqrt{r/E_0}$, attained exactly by tight frames satisfying
$QQ^\top=(E_0/r)I_r$.
\end{theorem}

Our estimator uses the same four named standard pulses as coordinate anchors,
giving $Q_m=I_4$ separately in every family. This gauge convention aligns the
meanings of the four axes. Cross-family equality of target-action coefficients
remains the separate hypothesis $H_{\rm share}$. The convention also removes an explicit
$Q_m^\dagger$ factor from operator recovery. With an independently fixed
action metric and energy budget, \cref{thm:pulse-opt} selects tight-frame
interventions. With common
pulse strength $h$, the stored differences
$R=h\widehat Y$ change only the global scale. The empirical state is
\begin{equation}
 z_m(x)=\big[\widehat c_m(x);\operatorname{vec}R_m(x)\big]\in\R^{25}.
 \label{eq:lstate}
\end{equation}
Every pulse branches from an independent copy of the checkpoint.

\subsection{Estimating a Pooled Action under Family Heterogeneity}

Enumerate the source checkpoints as
$\mathcal S=\{s_1,\ldots,s_n\}$, where $s_i$ belongs to family $m(s_i)$, and
write $\delta_{s_i,u}=\beta_{m(s_i)}+\gamma_{m(s_i),u}$. Stack source operators
and responses as
\begin{equation}
 \begin{aligned}
 \mathcal B_S&=\begin{bmatrix}B_{s_1}\\ \vdots\\ B_{s_n}\end{bmatrix},
 &y_S(u)&=\mathcal B_S\bar a_u+v^{\rm het}_{S,u}+e_S,\\
 v^{\rm het}_{S,u}&=\begin{bmatrix}B_{s_1}\delta_{s_1,u}\\ \vdots\\
 B_{s_n}\delta_{s_n,u}\end{bmatrix},
 &G_S&=\mathcal B_S^\top\mathcal B_S=\sum_s B_s^\top B_s.
 \end{aligned}
 \label{eq:source-stack}
\end{equation}
The source-coverage radius is
$\underline\sigma_S=\sigma_r(\mathcal B_S)=\sqrt{\lambda_{\min}(G_S)}$, while
$H_{S,u}=\|v^{\rm het}_{S,u}\|_2$ measures the linear response-scale cost of imposing a pooled
coordinate on heterogeneous source families.

\begin{theorem}[Approximate pooled-coordinate recovery under bounded heterogeneity]
\label{thm:source-coordinate}
Assume $\underline\sigma_S>0$, $\|e_S\|_2\leq\epsilon_S$,
$\|\bar a_u\|_2\leq A$, and
$\|v^{\rm het}_{S,u}\|_2\leq H_{S,u}$. With exact source operators,
$\widehat a_u=\mathcal B_S^\dagger y_S(u)$ obeys
$\|\widehat a_u-\bar a_u\|_2\leq(H_{S,u}+\epsilon_S)/\underline\sigma_S$.
If $\|\widehat{\mathcal B}_S-\mathcal B_S\|_2\leq\tau_S<\underline\sigma_S$, then
\begin{equation}
 \widehat a_u=\widehat{\mathcal B}_S^\dagger y_S(u),\qquad
 \|\widehat a_u-\bar a_u\|_2
 \leq\frac{\tau_S A+H_{S,u}+\epsilon_S}{\underline\sigma_S-\tau_S}.
 \label{eq:coordinate-bound}
\end{equation}
Under $H_{\rm share}$, $H_{S,u}=0$ and the original shared-coordinate bound is
recovered as a special case. The converse need not hold: $H_{S,u}$ measures
only the linear response-scale mismatch on the observed source operators.
\end{theorem}

Adding a source checkpoint changes the coverage Gram from $G_S$ to
$G_S+B_{\rm new}^\top B_{\rm new}$, so $\lambda_{\min}$ is nondecreasing. A
full-column-rank new source increases the lower bound by at least
$\sigma_r(B_{\rm new})^2$. Coverage alone, however, does not establish
sharedness: the added family can also increase $H_{S,u}$. Additional sources
tighten the complete bound only when their coverage gain outpaces operator,
response, and coordinate-heterogeneity errors.

\subsection{Prospective Transfer to an Unseen Family}

Let $\star$ denote a previously unseen family. Its target-independent pulses
give $\widehat B_\star$, while source-family target responses give
$\widehat a_u$, an estimate of the pooled coordinate $\bar a_u$. The prediction is
$\widehat r_\star(u)=\widehat B_\star\widehat a_u$.

\begin{theorem}[End-to-end cross-family transfer]
\label{thm:end-to-end}
Let $Q_\star$ be expressed in the same fixed named-pulse gauge, have full row
rank, and let $\widehat B_\star=\widehat Y_\star Q_\star^\dagger$. Assume
$\|\bar a_u\|_2\leq A$, $a_{\star,u}=\bar a_u+\delta_{\star,u}$ with
$\|\delta_{\star,u}\|_2\leq D_{\star,u}$,
$\|r_\star^h(u)-B_\star a_{\star,u}\|_2\leq\epsilon_{\star,u}$, target pulse
perturbation $\|E_\star\|_2\leq\eta_\star$, and the source conditions of
\cref{thm:source-coordinate}. Then
\begin{equation}
 \|\widehat r_\star(u)-r_\star^h(u)\|_2\leq\Delta_{\star,u},
 \label{eq:e2e}
\end{equation}
where
\begin{equation}
 \begin{aligned}
 \Delta_{\star,u}={}&\frac{\eta_\star A}{\sigma_r(Q_\star)}+\epsilon_{\star,u}\\
 &+\|B_\star\|_2D_{\star,u}\\
 &+\left(\|B_\star\|_2+\frac{\eta_\star}{\sigma_r(Q_\star)}\right)
 \frac{\tau_S A+H_{S,u}+\epsilon_S}{\underline\sigma_S-\tau_S}.
 \end{aligned}
 \label{eq:e2e-expanded}
\end{equation}
\end{theorem}

Equation \eqref{eq:e2e-expanded} separates target identification, target
locality, unseen-family coordinate bias, and source heterogeneity. Appendix
\ref{app:proofs} gives the full
proofs, a finite-repeat version, and the ridge-regularized source estimator used
by the operator readout. For that estimator, the same theorem holds after
replacing the last fraction in Eq. \eqref{eq:e2e-expanded} by the ridge
coordinate radius in Eq. \eqref{eq:ridge-coordinate}. If scoring uses a
realized response $\widetilde r_\star(u)=r_\star^h(u)+\omega_u$ with
$\|\omega_u\|_2\leq\kappa_u$, its total radius is
$\overline\Delta_{\star,u}=\Delta_{\star,u}+\kappa_u$.

\begin{corollary}[Operator direction and action recovery]
\label{cor:decisions}
If $|r_{\star,j}^h(u)|>\Delta_{\star,u}$, then
$\operatorname{sign}\widehat r_{\star,j}(u)=
\operatorname{sign}r_{\star,j}^h(u)$. For a finite candidate set with
$J(u)=h v^\top r_\star^h(u)-\operatorname{cost}(u)$, simultaneous radii imply
\begin{equation}
 J(u^\star)-J(\widehat u)
 \leq h\|v\|_2\big(\Delta_{\star,u^\star}
 +\Delta_{\star,\widehat u}\big).
 \label{eq:action-regret}
\end{equation}
If the true best-versus-runner-up utility margin exceeds
$h\|v\|_2(\Delta_{\star,u^\star}+\max_{u\neq u^\star}\Delta_{\star,u})$,
then $\widehat u=u^\star$.
Balanced-accuracy, normalized-regret, and realized-response versions appear in
Appendix \ref{app:proofs}.
\end{corollary}

\subsection{Two Pulse Readouts and Task-Matched Endpoints}
\label{sec:endpoint_policy}

\paragraph{Direct pulse readout.}
Let $p(s)=\operatorname{vec}R(s)$ denote the 20-dimensional pulse block of the
L-State at checkpoint $s$. For each target action, this readout fits a source-only
multi-output ridge map from $p(s)$ to the five-dimensional target response.
It lets the data learn an unconstrained linear combination of the measured
pulse responses.

\paragraph{Operator readout.}
We ridge-fit a target coordinate from source responses and predict
$\widehat y_{\rm op}(s,u)=R(s)\widehat a_u$. This preserves a zero point and
constrains predictions to the local operator geometry measured by the same
L-State pulse block. This is the structured predictor covered by Theorem
\ref{thm:end-to-end}.

\paragraph{Full L-State readout.}
For each action, a multi-output ridge map also predicts the response from the
complete 25-dimensional state $z(s)=[c(s);p(s)]$. All regularization and
standardization choices for the direct, operator, and full-state readouts use
complete source
trajectories. Appendix \ref{app:protocol} gives the objectives.

\paragraph{Shared-coordinate hypothesis audit.}
After all five family labels are open, we fit one coordinate per observed
family--action cell and apply Eq. \eqref{eq:coordinate-decomposition}. Here the
five observed families have equal weight, so their retrospective mean
$\bar a_u^{(5)}$ is distinct from the prospective source-reference mean.
Appendix \ref{app:sharedness_audit} defines the common-scale pooling error
$D_u$, shared-signal fraction $S_u$, and dead-zone error $D_u^{\rm DZ}$, and
gives the exact difference, retrospective equivalence, and predictive tests.
Those tests use unpenalized least squares; ridge is used only for coordinate
stability summaries. Thus $H_{\rm share}$ is tested directly.

\paragraph{Action-wise pulse selector.}
The first sealed family reveals that the source-calibrated operator can be
action-heterogeneous. We therefore compare the direct and operator readouts
separately for each semantic action using the three held-out development-family
folds. The direct readout is selected after a strict response-RMSE win in at
least two of the three folds; ties retain the operator readout. This rule selects
the operator readout for mathematics and reading and the direct readout for code
and science. The selector uses development-family responses
and is serialized before the second sealed family's target branches begin.
The deployed endpoint policy uses this selector for continuous response, the
operator readout for direction, and the full-state map for utility-conditioned
action scoring.

\paragraph{Empirical consequences.}
The theorem chain calls for a capability-only reference, pulse reliability and
geometry checks, direct sharedness diagnostics, and end-to-end response, sign,
and utility tests. The protocol evaluates them in that order and freezes each
held-out family's predictions before opening its target-action responses.

\section{Theory-Guided Experimental Design}

The experiments compare the two pulse readouts before testing the structured
transfer chain. Development measures the capability gap and the aggregate gains
of the direct and operator readouts; the first seal tests the composed operator
prediction and exposes an action-specific calibration inversion; repeat
interventions locate that failure; and the second seal tests the frozen
action-wise response route.

\subsection{Development Families and Sealed Families}

The development study uses Qwen2.5-7B-Instruct \citep{qwen2025report},
Mistral-7B-Instruct-v0.3 \citep{jiang2023mistral,mistralai2024mistralv03}, and
OLMo-2-1124-7B-Instruct \citep{olmo2025furious}. Each family contributes three
frozen seeds. Every seed starts from a zero-output rank-4 LoRA adapter
\citep{hu2022lora} and follows eight balanced state-generation episodes,
yielding checkpoints $t00$ through $t08$. The development set therefore contains
27 states per family and 81 states in total.

After the initial endpoint rule is fixed, we apply the same state construction
and training contract to GLM-4-9B-0414 \citep{zhipuai2025glm49b0414}. This first
sealed family adds 27 held-out states from three new trajectories. The GLM
science-action inversion then motivates the action-wise pulse selector above. We fit
its action-wise choices exclusively on the three development families and seal
the resulting policy before running Granite-3.1-8B-Instruct at revision
\texttt{4009206d5fc9} \citep{ibmgranite2024granite31}. Granite supplies another
27 states from three fresh trajectories. The complete study covers 135 states
from five 7--9B instruction-tuned model families.

\subsection{Data Separation and Training Contract}

Data have disjoint roles. Four deterministic synthetic tasks generate state
trajectories. Four different synthetic tasks provide the standard pulses:
schema mapping, two-step rule chaining, symbolic rewriting, and table
aggregation; this pulse set is synthetic and disjoint from the target benchmark
tasks. Semantic target actions
train on 32 examples from GSM8K \citep{cobbe2021gsm8k}, MBPP
\citep{austin2021mbpp}, SciQ \citep{welbl2017sciq}, or BoolQ
\citep{clark2019boolq}. Capabilities are negative clipped token NLL on disjoint
held-out examples from those four tasks plus WikiText-2
\citep{merity2016wikitext}. Exact normalized-text hashes are disjoint across
state generation, pulse training, target training, and capability evaluation.

All primary pulse and target branches freeze base weights and train rank-4 LoRA
on attention query and value projections for eight AdamW steps, consuming 32
examples. The target-early reference reads a target-specific intermediate after two steps and eight
examples, while locality calibration deliberately includes four- and
sixteen-step variants. Each branch starts with reset optimizer moments and
restores an identical checkpoint hash. This short-LoRA protocol measures
immediate training response under one shared optimization contract.

\subsection{Family-Level Gating and Baselines}

In development, each outer fold trains on two model families and tests on the
third. The program serializes predictions and hashes before opening the fold's
label vault, while source-trajectory inner cross-validation selects
hyperparameters. The complete comparison includes a source-mean reference,
\caponly{}, a direct pulse readout, scalar training statistics, a
trajectory/index-matched source state, and the \ophead{} readout.
\textsc{Target Early} provides a separate target-specific reference.

Each sealed stage freezes its reporting rule, all 540 core response predictions,
and all 810 utility-conditioned action choices before target-action training.
The Granite stage additionally freezes 216 selector-extension rows that copy
the chosen direct or operator response and direction predictions. Both Granite freeze
receipts record zero target branches. The target vault then opens and scores
five core methods over 27 states, four actions, five response coordinates, and
six utilities. Current capability is the common reference. The direct readout
fits the 20-dimensional L-State pulse block, while the operator readout applies
the structured factorization to the same block. The matched-state reference averages
development responses at the same trajectory seed, state index, and action
identity.

\subsection{Metrics and Uncertainty}

Continuous response error is standardized using only the source-family mean
and population standard deviation for each outer fold, action, and coordinate;
the resulting RMSE is dimensionless. The later five-family sharedness audit
instead uses one common three-development-family scale for every family so that
all deviations have the same ruler. We report MSE gain as
$1-\mathrm{RMSE}_m^2/\mathrm{RMSE}_{\mathrm{cap}}^2$. Direction is balanced accuracy on response
components outside a coordinate-specific dead zone estimated from three-repeat
pulse reliability. We also report median response cosine.

For decisions, six fixed utility vectors---five one-hot capabilities and one
balanced vector---choose among the four semantic actions. We report top-1 action
accuracy and realized normalized regret. Because each semantic action has one
preregistered training realization, the decision metric measures realized
performance under the fixed training protocol.

The development analysis uses 10,000 family-then-trajectory bootstrap
resamples for full-state readout endpoints and 1,000 aligned-pulse permutations with
full source refitting. Each sealed-family core endpoint analysis uses 10,000
whole-trajectory resamples over its three trajectories. GLM direction is
evaluated over 184 components selected by source-frozen dead zones; Granite
uses the same frozen coordinate thresholds. The action-wise pulse selector is
summarized by frozen pooled and action-macro point estimates, while continuous
response retains all scored components. Four
additional GLM science-action responses per state test repeat stability, and a
separate nine-state study evaluates four-, eight-, and sixteen-step exposure.
Fresh processes reproduce the canonical development, GLM, and Granite analysis
hashes.

\subsection{Retrospective Five-Family Sharedness Audit}

This audit begins only after all five target-label vaults are open and leaves
every sealed prediction unchanged. Full-rank frozen action matrices recover the
  pulse operators algebraically, replaying operator-readout predictions below
  $10^{-9}$ error.
Primary comparisons use one development-frozen response scale, complete
trajectory blocks, and exact sign flips over the 15 observed
family--trajectory clusters. The audit is conditional on these five families,
fixed pulse measurements, and one training realization; Appendix
\ref{app:sharedness_audit} gives recovery, conditioning, resampling, and
sensitivity details.

\section{Results}

\subsection{L-State Readouts in Three-Family Development}

Development first asks whether the L-State pulse block carries transferable
response information and how readout structure changes the endpoint. The direct
readout uses the pulse block as features, the operator readout preserves its
matrix geometry, and current capability provides the reference. Across the
three LOFO folds, the direct and operator readouts have nearly identical
aggregate response error: RMSE 7.772 and 7.777 versus 9.987 for capability. Each
reduces source-standardized MSE by 39.4\%. The direct readout has the lowest
response-RMSE point estimate, while the operator readout has the highest
aggregate sign balanced accuracy, 0.675 versus 0.602 for capability. The full
L-State readout separately reduces
realized normalized regret from 0.472 to 0.248. The complete method table
appears in Appendix \ref{app:development_results}; aligned-pulse refits exceed
the registered per-fold null 95th percentile on Qwen and Mistral.

\subsection{Sealed GLM Test of the Operator Readout}
\label{sec:prospective}

The first sealed test asks whether the structured readout transfers without
target-action responses. We freeze 540 response predictions and 810 action
choices before opening GLM labels. The direct readout reaches RMSE 0.712, a 71.8\% MSE
reduction relative to capability. The operator readout reaches RMSE 0.624, a
78.3\% reduction, and has
the highest sign balanced-accuracy point estimate among the five frozen core
methods, increasing from 0.366 for capability to 0.754. The full L-State
readout lowers realized normalized regret from 0.592 to 0.434; the matched-state
reference attains 0.410. Complete frozen-method and
resampling comparisons appear in Appendices \ref{app:sealed_summary} and
\ref{app:prospective_details}.

\subsection{From an Operator Failure to Action-Wise Readout Selection}

The GLM science action separates the measured representation from the
structure imposed by the operator factorization. The operator readout has sign
BA zero, whereas the direct readout reaches 1.0 on the
determinate science coordinates. Four additional target realizations preserve
the target sign in all 27 states, while source-fit operator coordinates span
$-0.555$ to 0.322. This locates the instability on the source-coordinate side
without identifying a unique cause. Motivated by this sealed failure, we fit an
action-wise selector using only the three development families. It retains the
operator readout for mathematics and reading and chooses the direct readout for
code and science. The mapping is frozen before Granite labels open.

\subsection{Sealed Granite Test of Action-Wise Readout Selection}
\label{sec:granite}

Granite tests the frozen route on a second family. The direct readout has the
lowest point-estimate RMSE among the single readouts at 0.544; the operator
readout reaches 0.634, and the development-fitted selector reaches 0.554,
versus 1.172 for capability. The selector therefore preserves most of the
direct readout's aggregate response gain while improving over both capability
and the globally applied operator readout. The operator readout produces
non-degenerate action-wise direction estimates with aggregate sign BA 0.657,
compared with 0.654 for capability. The full L-State readout lowers regret from
0.291 to 0.217 and raises top-1 from 53.1\% to 63.6\%. Frozen receipts and
complete comparisons appear in Appendices \ref{app:sealed_summary} and
\ref{app:granite_details}.

\subsection{Testing the Operator Readout's Action Coordinate}
\label{sec:sharedness_results}

The direct and operator readouts use the same pulse measurement but place
different constraints on it. The operator readout estimates one pooled action
coordinate from source families, so we test that structure directly after all
five family labels are open. The
common-scale pooling error ranges from $D_u=0.342$ for mathematics to 1.038 for
reading; code has the highest shared-signal fraction ($S_u=0.800$), science the
lowest ($S_u=0.269$), and reading has negative mean coordinate cosine. Appendix
Table \ref{tab:sharedness_full} gives the full action-wise audit.

At level 0.05, exact sharing is rejected for code, science, and reading.
Mathematics remains unresolved by the difference test; all four
$D_u^{\rm DZ}$ values exceed one, so no action meets the retrospective
one-dead-zone equivalence criterion. Joint tests detect family main effects,
total deviation, and the family--action interaction increment (all
$p\leq3.05\times10^{-4}$). Shared, family-main, and full family--action models
have held-trajectory RMSE 0.824, 0.802, and 0.750; the full model reduces MSE by
17.3\% overall and 48.0\% for science. This fixed-family diagnosis explains how
family and action deviations affect the operator structure. Appendix
\ref{app:sharedness_audit} gives intervals, decomposition, and sensitivities.

\section{Conclusion}

Target-independent micro-interventions turn a checkpoint into a measured
training-response state. A direct readout learns a flexible map from its L-State
pulse block; an operator readout factors the same block into local response
geometry and an action coordinate. Both cut development MSE by 39.4\% relative
to capability alone. On GLM, the operator readout reaches RMSE 0.624 and sign
balanced accuracy 0.754; on Granite, the direct readout reaches RMSE 0.544 and
the frozen action-wise selector reaches 0.554, versus 1.172 for capability
alone. The theory supports pulse identification and conditional operator
transfer, while source families choose readouts by action and endpoint. The
five-family audit quantifies deviations in the pooled operator coordinate, and the full L-State
readout adds action-ranking gains. The central result is a transferable
intervention measurement supporting complementary direct and structured
readouts.

\subsection*{AI use statement}

Generative AI tools assisted with conceptual and theoretical development,
mathematical claims and proof drafting, experimental design, implementation,
execution, result analysis, figure creation, literature review, and manuscript
drafting, editing, and formatting. The authors reviewed all AI-assisted work,
tested the code, audited the proofs, replayed the frozen experimental evidence,
and checked the cited sources. The authors take responsibility for the final
content, including text, claims, code, and artifacts produced with generative-AI
assistance.

\subsection*{Reproducibility statement}

Sections 4--5 specify the data separation, training contract, frozen prediction
protocol, metrics, and replay procedure. Appendices A--K provide complete
proofs, experimental details, development and sealed-family tables, repeat and
duration studies, probe-cost accounting, and artifact hashes. Each headline
result is linked to a frozen evidence file and a fresh-process replay receipt.

\bibliography{references}

\begin{thebibliography}{30}
\providecommand{\natexlab}[1]{#1}
\providecommand{\url}[1]{\texttt{#1}}
\expandafter\ifx\csname urlstyle\endcsname\relax
  \providecommand{\doi}[1]{doi: #1}\else
  \providecommand{\doi}{doi: \begingroup \urlstyle{rm}\Url}\fi

\bibitem[Austin et~al.(2021)Austin, Odena, Nye, Bosma, Michalewski, Dohan,
  Jiang, Cai, Terry, Le, and Sutton]{austin2021mbpp}
Jacob Austin, Augustus Odena, Maxwell Nye, Maarten Bosma, Henryk Michalewski,
  David Dohan, Ellen Jiang, Carrie Cai, Michael Terry, Quoc Le, and Charles
  Sutton.
\newblock Program synthesis with large language models, 2021.
\newblock URL \url{https://arxiv.org/abs/2108.07732}.

\bibitem[Clark et~al.(2019)Clark, Lee, Chang, Kwiatkowski, Collins, and
  Toutanova]{clark2019boolq}
Christopher Clark, Kenton Lee, Ming-Wei Chang, Tom Kwiatkowski, Michael
  Collins, and Kristina Toutanova.
\newblock {BoolQ}: Exploring the surprising difficulty of natural yes/no
  questions.
\newblock In \emph{Proceedings of NAACL-HLT 2019}, pp.\  2924--2936.
  Association for Computational Linguistics, 2019.
\newblock \doi{10.18653/v1/N19-1300}.
\newblock URL \url{https://aclanthology.org/N19-1300/}.

\bibitem[Cobbe et~al.(2021)Cobbe, Kosaraju, Bavarian, Chen, Jun, Kaiser,
  Plappert, Tworek, Hilton, Nakano, et~al.]{cobbe2021gsm8k}
Karl Cobbe, Vineet Kosaraju, Mohammad Bavarian, Mark Chen, Heewoo Jun, Lukasz
  Kaiser, Matthias Plappert, Jerry Tworek, Jacob Hilton, Rei Nakano, et~al.
\newblock Training verifiers to solve math word problems, 2021.
\newblock URL \url{https://arxiv.org/abs/2110.14168}.

\bibitem[Gulrajani \& Lopez-Paz(2021)Gulrajani and
  Lopez-Paz]{gulrajani2021domainbed}
Ishaan Gulrajani and David Lopez-Paz.
\newblock In search of lost domain generalization.
\newblock In \emph{International Conference on Learning Representations}, 2021.
\newblock URL \url{https://openreview.net/forum?id=lQdXeXDoWtI}.

\bibitem[Hewitt et~al.(2021)Hewitt, Ethayarajh, Liang, and
  Manning]{hewitt2021conditional}
John Hewitt, Kawin Ethayarajh, Percy Liang, and Christopher~D. Manning.
\newblock Conditional probing: Measuring usable information beyond a baseline.
\newblock In \emph{Proceedings of the 2021 Conference on Empirical Methods in
  Natural Language Processing}, pp.\  1626--1639. Association for Computational
  Linguistics, 2021.
\newblock \doi{10.18653/v1/2021.emnlp-main.122}.
\newblock URL \url{https://aclanthology.org/2021.emnlp-main.122/}.

\bibitem[Hu et~al.(2022)Hu, Shen, Wallis, Allen-Zhu, Li, Wang, Wang, and
  Chen]{hu2022lora}
Edward~J. Hu, Yelong Shen, Phillip Wallis, Zeyuan Allen-Zhu, Yuanzhi Li, Shean
  Wang, Lu~Wang, and Weizhu Chen.
\newblock {LoRA}: Low-rank adaptation of large language models.
\newblock In \emph{International Conference on Learning Representations}, 2022.
\newblock URL \url{https://openreview.net/forum?id=nZeVKeeFYf9}.

\bibitem[{IBM Granite Team}(2024)]{ibmgranite2024granite31}
{IBM Granite Team}.
\newblock {Granite-3.1-8B-Instruct}.
\newblock Hugging Face model repository, 2024.
\newblock URL \url{https://huggingface.co/ibm-granite/granite-3.1-8b-instruct}.
\newblock Model release; frozen revision 4009206d5fc9.

\bibitem[Ilharco et~al.(2023)Ilharco, Ribeiro, Wortsman, Gururangan, Schmidt,
  Hajishirzi, and Farhadi]{ilharco2023task}
Gabriel Ilharco, Marco~Tulio Ribeiro, Mitchell Wortsman, Suchin Gururangan,
  Ludwig Schmidt, Hannaneh Hajishirzi, and Ali Farhadi.
\newblock Editing models with task arithmetic.
\newblock In \emph{International Conference on Learning Representations}, 2023.
\newblock URL \url{https://openreview.net/forum?id=6t0Kwf8-jrj}.

\bibitem[Ilyas et~al.(2022)Ilyas, Park, Engstrom, Leclerc, and
  Madry]{ilyas2022datamodels}
Andrew Ilyas, Sung~Min Park, Logan Engstrom, Guillaume Leclerc, and Aleksander
  Madry.
\newblock Datamodels: Understanding predictions with data and data with
  predictions.
\newblock In \emph{Proceedings of the 39th International Conference on Machine
  Learning}, volume 162 of \emph{Proceedings of Machine Learning Research},
  pp.\  9525--9587. PMLR, 2022.
\newblock URL \url{https://proceedings.mlr.press/v162/ilyas22a.html}.

\bibitem[Jiang et~al.(2023)Jiang, Sablayrolles, Mensch, Bamford, Chaplot,
  de~las Casas, Bressand, Lengyel, Lample, Saulnier, et~al.]{jiang2023mistral}
Albert~Q. Jiang, Alexandre Sablayrolles, Arthur Mensch, Chris Bamford,
  Devendra~Singh Chaplot, Diego de~las Casas, Florian Bressand, Gianna Lengyel,
  Guillaume Lample, Lucile Saulnier, et~al.
\newblock Mistral 7b, 2023.
\newblock URL \url{https://arxiv.org/abs/2310.06825}.

\bibitem[Koh \& Liang(2017)Koh and Liang]{koh2017influence}
Pang~Wei Koh and Percy Liang.
\newblock Understanding black-box predictions via influence functions.
\newblock In \emph{Proceedings of the 34th International Conference on Machine
  Learning}, volume~70 of \emph{Proceedings of Machine Learning Research}, pp.\
   1885--1894. PMLR, 2017.
\newblock URL \url{https://proceedings.mlr.press/v70/koh17a.html}.

\bibitem[Lin et~al.(2024)Lin, Huang, Ye, Chen, Wang, Li, Ma, Wan, Zou, and
  Liang]{lin2024rectified}
Haowei Lin, Baizhou Huang, Haotian Ye, Qinyu Chen, Zihao Wang, Sujian Li,
  Jianzhu Ma, Xiaojun Wan, James Zou, and Yitao Liang.
\newblock Selecting large language model to fine-tune via rectified scaling
  law.
\newblock In \emph{Proceedings of the 41st International Conference on Machine
  Learning}, volume 235 of \emph{Proceedings of Machine Learning Research},
  pp.\  30080--30107. PMLR, 2024.
\newblock URL \url{https://proceedings.mlr.press/v235/lin24j.html}.

\bibitem[Luo et~al.(2026)Luo, Long, Wang, Duan, Lin, Xu, Luo, Yang, and
  Tang]{luo2026tuneahead}
Yuxiang Luo, Haonan Long, Chen Wang, Qiqi Duan, Xiaotian Lin, Yanwei Xu, Yuyu
  Luo, Weikai Yang, and Nan Tang.
\newblock {TuneAhead}: Predicting fine-tuning performance before full training
  begins, 2026.
\newblock URL \url{https://arxiv.org/abs/2606.17660}.
\newblock Accepted at ICML 2026.

\bibitem[Merity et~al.(2017)Merity, Xiong, Bradbury, and
  Socher]{merity2016wikitext}
Stephen Merity, Caiming Xiong, James Bradbury, and Richard Socher.
\newblock Pointer sentinel mixture models.
\newblock In \emph{International Conference on Learning Representations}, 2017.
\newblock URL \url{https://openreview.net/forum?id=Byj72udxe}.

\bibitem[{Mistral AI}(2024)]{mistralai2024mistralv03}
{Mistral AI}.
\newblock {Mistral-7B-Instruct-v0.3}.
\newblock Hugging Face model repository, 2024.
\newblock URL \url{https://huggingface.co/mistralai/Mistral-7B-Instruct-v0.3}.

\bibitem[Munn \& Wei(2025)Munn and Wei]{munn2025bayesian}
Michael Munn and Susan Wei.
\newblock A {B}ayesian model selection criterion for selecting pretraining
  checkpoints.
\newblock In \emph{Proceedings of the 42nd International Conference on Machine
  Learning}, volume 267 of \emph{Proceedings of Machine Learning Research},
  pp.\  45256--45271. PMLR, 2025.
\newblock URL \url{https://proceedings.mlr.press/v267/munn25a.html}.

\bibitem[Nguyen et~al.(2020)Nguyen, Hassner, Seeger, and
  Archambeau]{nguyen2020leep}
Cuong Nguyen, Tal Hassner, Matthias Seeger, and Cedric Archambeau.
\newblock {LEEP}: A new measure to evaluate transferability of learned
  representations.
\newblock In \emph{Proceedings of the 37th International Conference on Machine
  Learning}, volume 119 of \emph{Proceedings of Machine Learning Research},
  pp.\  7294--7305. PMLR, 2020.
\newblock URL \url{https://proceedings.mlr.press/v119/nguyen20b.html}.

\bibitem[{Qwen Team}(2025)]{qwen2025report}
{Qwen Team}.
\newblock {Qwen2.5} technical report, 2025.
\newblock URL \url{https://arxiv.org/abs/2412.15115}.

\bibitem[Rinaldi et~al.(2026)Rinaldi, Panariello, Salici, Porrello, and
  Calderara]{rinaldi2026theseus}
Filippo Rinaldi, Aniello Panariello, Giacomo Salici, Angelo Porrello, and
  Simone Calderara.
\newblock Transporting task vectors across different architectures without
  training, 2026.
\newblock URL \url{https://arxiv.org/abs/2602.12952}.
\newblock Accepted at ICML 2026.

\bibitem[Ruan et~al.(2024)Ruan, Maddison, and Hashimoto]{ruan2024observational}
Yangjun Ruan, Chris~J. Maddison, and Tatsunori Hashimoto.
\newblock Observational scaling laws and the predictability of language model
  performance.
\newblock In \emph{Advances in Neural Information Processing Systems},
  volume~37, 2024.
\newblock \doi{10.52202/079017-0506}.

\bibitem[Wagenmaker \& Jamieson(2020)Wagenmaker and
  Jamieson]{wagenmaker2020active}
Andrew Wagenmaker and Kevin Jamieson.
\newblock Active learning for identification of linear dynamical systems.
\newblock In \emph{Proceedings of the Thirty Third Conference on Learning
  Theory}, volume 125 of \emph{Proceedings of Machine Learning Research}, pp.\
  3487--3582. PMLR, 2020.
\newblock URL \url{https://proceedings.mlr.press/v125/wagenmaker20a.html}.

\bibitem[Walsh et~al.(2025)Walsh, Soldaini, Groeneveld, Lo, Arora, Bhagia, Gu,
  Huang, Jordan, et~al.]{olmo2025furious}
Evan~Pete Walsh, Luca Soldaini, Dirk Groeneveld, Kyle Lo, Shane Arora, Akshita
  Bhagia, Yuling Gu, Shengyi Huang, Matt Jordan, et~al.
\newblock 2 {OLMo} 2 furious.
\newblock In \emph{Second Conference on Language Modeling}, 2025.
\newblock URL \url{https://openreview.net/forum?id=2ezugTT9kU}.

\bibitem[Wang(2026)]{wang2026unified}
Mian Wang.
\newblock Training, learning and inference: Unified dynamics of neural systems,
  2026.
\newblock URL \url{https://arxiv.org/abs/2608.20965}.
\newblock Concurrent preprint.

\bibitem[Welbl et~al.(2017)Welbl, Liu, and Gardner]{welbl2017sciq}
Johannes Welbl, Nelson~F. Liu, and Matt Gardner.
\newblock Crowdsourcing multiple choice science questions.
\newblock In \emph{Proceedings of the 3rd Workshop on Noisy User-generated
  Text}, pp.\  94--106. Association for Computational Linguistics, 2017.
\newblock \doi{10.18653/v1/W17-4413}.
\newblock URL \url{https://aclanthology.org/W17-4413/}.

\bibitem[Wu et~al.(2026)Wu, Zhao, Wang, Guo, Wang, and He]{wu2026llmdna}
Zhaomin Wu, Haodong Zhao, Ziyang Wang, Jizhou Guo, Qian Wang, and Bingsheng He.
\newblock {LLM DNA}: Tracing model evolution via functional representations.
\newblock In \emph{International Conference on Learning Representations}, 2026.

\bibitem[Xia et~al.(2023)Xia, Artetxe, Zhou, Lin, Pasunuru, Chen, Zettlemoyer,
  and Stoyanov]{xia2023trajectories}
Mengzhou Xia, Mikel Artetxe, Chunting Zhou, Xi~Victoria Lin, Ramakanth
  Pasunuru, Danqi Chen, Luke Zettlemoyer, and Veselin Stoyanov.
\newblock Training trajectories of language models across scales.
\newblock In \emph{Proceedings of the 61st Annual Meeting of the Association
  for Computational Linguistics}, pp.\  13711--13738. Association for
  Computational Linguistics, 2023.
\newblock \doi{10.18653/v1/2023.acl-long.767}.
\newblock URL \url{https://aclanthology.org/2023.acl-long.767/}.

\bibitem[Xia et~al.(2024)Xia, Malladi, Gururangan, Arora, and
  Chen]{xia2024less}
Mengzhou Xia, Sadhika Malladi, Suchin Gururangan, Sanjeev Arora, and Danqi
  Chen.
\newblock {LESS}: Selecting influential data for targeted instruction tuning.
\newblock In \emph{Proceedings of the 41st International Conference on Machine
  Learning}, volume 235 of \emph{Proceedings of Machine Learning Research},
  pp.\  54104--54132. PMLR, 2024.
\newblock URL \url{https://proceedings.mlr.press/v235/xia24c.html}.

\bibitem[{Zhipu AI}(2025)]{zhipuai2025glm49b0414}
{Zhipu AI}.
\newblock {GLM-4-9B-0414}.
\newblock Hugging Face model repository, 2025.
\newblock URL \url{https://huggingface.co/zai-org/GLM-4-9B-0414}.
\newblock Model release; revision 645b8482494e31b6b752272bf7f7f273ef0f3caf.

\bibitem[Zhu et~al.(2022)Zhu, Shahtalebi, and Rudzicz]{zhu2022probing}
Zining Zhu, Soroosh Shahtalebi, and Frank Rudzicz.
\newblock Predicting fine-tuning performance with probing.
\newblock In \emph{Proceedings of the 2022 Conference on Empirical Methods in
  Natural Language Processing}, pp.\  11534--11547. Association for
  Computational Linguistics, 2022.
\newblock \doi{10.18653/v1/2022.emnlp-main.793}.
\newblock URL \url{https://aclanthology.org/2022.emnlp-main.793/}.

\bibitem[Zhuang et~al.(2025)Zhuang, Wu, Wen, Li, Jiao, and
  Ramchandran]{zhuang2025embedllm}
Richard Zhuang, Tianhao Wu, Zhaojin Wen, Andrew Li, Jiantao Jiao, and Kannan
  Ramchandran.
\newblock {EmbedLLM}: Learning compact representations of large language
  models.
\newblock In \emph{International Conference on Learning Representations}, 2025.

\end{thebibliography}
\bibliographystyle{iclr2027_conference}

\appendix

\section{Sealed-Family Summary Tables}
\label{app:sealed_summary}

\begin{table}[H]
\caption{\textbf{First sealed-family evaluation.} Predictions and action
choices are frozen before GLM target-action labels are opened. The direct
readout uses the L-State pulse block as features; the operator readout is the
frozen structured response/direction predictor, and the full L-State readout is
the frozen action predictor.
Sign BA is coordinate-macro balanced accuracy over source-dead-zone-determinate
components. The direct and matched-state methods are endpoint-matched response
and action comparators.}
\label{tab:prospective}
\centering
\scriptsize
\setlength{\tabcolsep}{3pt}
\begin{minipage}[t]{0.47\textwidth}
\centering
\textbf{Signed continuous response and direction}\\[3pt]
\begin{tabular}{lrrr}
\toprule
Method & RMSE $\downarrow$ & MSE gain $\uparrow$ & Sign BA $\uparrow$ \\
\midrule
Capability & 1.339 & 0.000 & 0.366 \\
Direct pulse & 0.712 & 0.718 & 0.571 \\
\textbf{Operator}$^\dagger$ & \textbf{0.624} & \textbf{0.783} & \textbf{0.754} \\
\bottomrule
\end{tabular}
\end{minipage}
\hfill
\begin{minipage}[t]{0.47\textwidth}
\centering
\textbf{Action selection}\\[3pt]
\begin{tabular}{lrrr}
\toprule
Method & Regret $\downarrow$ & Reduction $\uparrow$ & Top-1 $\uparrow$ \\
\midrule
Capability & 0.592 & 0.000 & 0.340 \\
\textbf{Matched state} & \textbf{0.410} & \textbf{0.307} & \textbf{0.432} \\
Full L-State$^\dagger$ & 0.434 & 0.267 & 0.395 \\
\bottomrule
\end{tabular}
\end{minipage}
\vspace{2pt}

\footnotesize $^\dagger$Frozen endpoint readout. MSE gain and regret reduction
use capability as the reference.
\end{table}

\begin{table}[H]
\caption{\textbf{Second sealed-family evaluation on Granite-3.1-8B.} The
action-wise selector is fitted only from the three development families and
frozen before Granite target branches; the operator readout is registered for
direction and the full \lstate{} readout for action ranking.
All rows are frozen before target-action labels are opened. Full five-method
metrics appear in Appendix \ref{app:granite_details}.}
\label{tab:granite}
\centering
\scriptsize
\setlength{\tabcolsep}{3pt}
\begin{minipage}[t]{0.31\textwidth}
\centering
\textbf{Continuous response}\\[3pt]
\resizebox{\linewidth}{!}{%
\begin{tabular}{lrr}
\toprule
Method & RMSE $\downarrow$ & MSE gain $\uparrow$ \\
\midrule
Capability & 1.172 & 0.0\% \\
Direct pulse & \textbf{0.544} & \textbf{78.4\%} \\
Operator & 0.634 & 70.7\% \\
\sourcegate{} & 0.554 & 77.7\% \\
\bottomrule
\end{tabular}}
\end{minipage}
\hfill
\begin{minipage}[t]{0.30\textwidth}
\centering
\textbf{Direction}\\[3pt]
\resizebox{\linewidth}{!}{%
\begin{tabular}{lr}
\toprule
Method & Sign BA $\uparrow$ \\
\midrule
Capability & 0.654 \\
Direct pulse & 0.535 \\
Operator$^\dagger$ & 0.657 \\
Full L-State & \textbf{0.664} \\
\bottomrule
\end{tabular}}
\end{minipage}
\hfill
\begin{minipage}[t]{0.34\textwidth}
\centering
\textbf{Action selection}\\[3pt]
\resizebox{\linewidth}{!}{%
\begin{tabular}{lrr}
\toprule
Method & Regret $\downarrow$ & Top-1 $\uparrow$ \\
\midrule
Capability & 0.291 & 0.531 \\
Direct pulse & 0.215 & 0.623 \\
Matched state & \textbf{0.212} & \textbf{0.648} \\
Full L-State$^\dagger$ & 0.217 & 0.636 \\
\bottomrule
\end{tabular}}
\end{minipage}
\vspace{2pt}

\footnotesize $^\dagger$Registered endpoint readout. MSE gain uses capability as the
reference. Selector RMSE pools all 540 scalar response components. Sign BA
pools actions within each capability coordinate, computes balanced accuracy
coordinate-wise, and macro-averages across coordinates. The matched-state reference uses aligned
source-state responses and therefore a stronger information set than the
deployment readouts.
\end{table}

\section{Core Proofs}
\label{app:proofs}

\subsection{Capability Level Sets}

\begin{proof}[Proof of \cref{thm:capability}]
Let $r=r_m^h(x,u)$, $r'=r_{m'}^h(x',u)$, and $y=f(c,u)$. The triangle
inequality gives
$\|r-r'\|_2\leq\|r-y\|_2+\|y-r'\|_2$, so at least one of the two errors is
at least $\|r-r'\|_2/2$.

For the characterization, if $r_m^h(x,u)=f(c_m(x),u)$, equal capabilities
imply equal responses. Conversely, if responses are constant on every
capability level set, define $f(c,u)$ as the common response of any checkpoint
in that level set. Constancy makes this definition independent of the chosen
representative.
\end{proof}

\subsection{Factorization from Smooth Dynamics}

\begin{proof}[Proof of \cref{thm:smooth-factor}]
Suppress $(m,x,u,h)$ and write $\Delta=\Delta_{m,u}^h(x,\xi)$. The integral
Taylor formula gives
\begin{equation}
 c_m(x+\Delta)-c_m(x)=Jc_m(x)\Delta+\rho(\Delta),
\end{equation}
where
\begin{align*}
 \|\rho(\Delta)\|_2
 &\leq\int_0^1
 \|Jc_m(x+t\Delta)-Jc_m(x)\|_2\,\|\Delta\|_2\,dt\\
 &\leq\tfrac12L_c\|\Delta\|_2^2.
\end{align*}
Taking expectations, dividing by $h$, and inserting
$\mathbb E\Delta/h=G_m(x)a_{m,u}+e_g$ yields
\[
 r_m^h(x,u)=Jc_m(x)G_m(x)a_{m,u}
 +Jc_m(x)e_g+\mathbb E\rho(\Delta)/h.
\]
The assumed bounds on $e_g$ and the second moment of $\Delta$ give
Eq. \eqref{eq:local-remainder}.
\end{proof}

For deterministic gradient flow with primitive losses
$L_{a_m}=\sum_{j=1}^r a_{m,j}L_j$, the update map at $x$ has columns
$-\nabla L_j(x)$. The theorem then recovers
$B_{:,j}=-Jc(x)\nabla L_j(x)$, with an $O(h\|a_m\|_2^2)$ remainder whenever
the induced capability velocity is locally Lipschitz.

\subsection{Pulse Identifiability and Design}

\begin{proof}[Proof of \cref{thm:pulse-ident}]
If $a_\star=Q\lambda$, then $Ba_\star=(BQ)\lambda$, so the pulse observations
determine the target response. If $a_\star\notin\operatorname{col}(Q)$, choose
$z\in\operatorname{null}(Q^\top)$ with $z^\top a_\star\neq0$ and a nonzero
$w\in\R^d$. The matrix $H=wz^\top$ obeys $HQ=0$ but
$Ha_\star\neq0$; hence $B$ and $B+H$ have identical pulse responses and
different target responses. Applying this statement to every $a_\star\in\R^r$
gives the full-row-rank characterization.

When $Q$ has full row rank, $QQ^\dagger=I_r$, and therefore
\[
 \widehat B-B=(BQ+E)Q^\dagger-B=EQ^\dagger.
\]
Submultiplicativity and $\|Q^\dagger\|_2=1/\sigma_r(Q)$ give both norm bounds in
Eq. \eqref{eq:ident}.
\end{proof}

\begin{proof}[Proof of \cref{thm:pulse-opt}]
The upper bound
$\|EQ^\dagger\|_2\leq\|E\|_2\|Q^\dagger\|_2$ is attained by choosing a
rank-one $E$ aligned with a leading left singular vector of $Q^\dagger$.
Thus the exact amplification is $\|Q^\dagger\|_2=1/\sigma_r(Q)$.

Let $\sigma_1\geq\cdots\geq\sigma_r>0$ be the singular values of $Q$.
Then
\[
 \sigma_r(Q)^2\leq\frac1r\sum_{j=1}^r\sigma_j(Q)^2
 =\frac{\|Q\|_F^2}{r}\leq\frac{E_0}{r}.
\]
Equality holds exactly when the energy budget is tight and all $r$ singular
values are equal, equivalently $QQ^\top=(E_0/r)I_r$. Taking reciprocals proves
the minimax statement.
\end{proof}

If each of $k$ pulses is independently repeated $n$ times, every coordinate of
the repeat-mean noise is mean-zero sub-Gaussian with variance proxy
$\sigma^2/n$, and
each pulse residual has norm at most $\epsilon_p$, then with probability at
least $1-\delta$,
\begin{equation}
 \|E\|_2\leq\sqrt{k}\epsilon_p+
 \sigma\sqrt{\frac{2dk\log(2dk/\delta)}{n}}.
 \label{eq:finite-repeat}
\end{equation}
This follows from a coordinatewise tail bound, a union bound over $dk$
entries, and $\|Z\|_2\leq\|Z\|_F$. Combining
Eqs. \eqref{eq:ident} and \eqref{eq:finite-repeat} supplies a finite-repeat
choice of the target-pulse term $\eta_\star$ in \cref{thm:end-to-end}; the
source target-response, target-action residual, and realized-response terms
retain their displayed radii.

\subsection{Source Coverage and Action Coordinates}

\begin{proof}[Proof of \cref{thm:source-coordinate}]
With exact operators,
$\widehat a_u-\bar a_u=\mathcal B_S^\dagger(v^{\rm het}_{S,u}+e_S)$, whose norm is at
most $(H_{S,u}+\epsilon_S)/\underline\sigma_S$.

For estimated operators, Weyl's inequality gives
$\sigma_r(\widehat{\mathcal B}_S)\geq\underline\sigma_S-\tau_S>0$. Rewriting the source
response as
\[
 y_S(u)=\widehat{\mathcal B}_S\bar a_u
 +(\mathcal B_S-\widehat{\mathcal B}_S)\bar a_u+v^{\rm het}_{S,u}+e_S
\]
and left-multiplying by $\widehat{\mathcal B}_S^\dagger$ yields
\[
 \widehat a_u-\bar a_u=\widehat{\mathcal B}_S^\dagger
 \big[(\mathcal B_S-\widehat{\mathcal B}_S)\bar a_u+v^{\rm het}_{S,u}+e_S\big].
\]
The pseudoinverse norm is at most $1/(\underline\sigma_S-\tau_S)$, proving
Eq. \eqref{eq:coordinate-bound}.
\end{proof}

The source-diversity statement follows from
$G_{S\cup\{\mathrm{new}\}}=G_S+B_{\mathrm{new}}^\top B_{\mathrm{new}}$ and
Weyl monotonicity for positive semidefinite matrices. In particular,
\[
 \lambda_{\min}(G_{S\cup\{\mathrm{new}\}})
 \geq\lambda_{\min}(G_S)+\sigma_r(B_{\mathrm{new}})^2.
\]
If source block $s$ has operator error at most $\beta_s$, the stacked
perturbation satisfies
$\tau_S\leq(\sum_s\beta_s^2)^{1/2}$. Applying
Eq. \eqref{eq:finite-repeat} to each source block and taking a joint event
therefore supplies the finite-repeat source term in Eq. \eqref{eq:e2e-expanded}.

Under the same source conditions, including
$\tau_S<\underline\sigma_S$, the ridge estimator in
Eq. \eqref{eq:operator} also admits an explicit source bound. For
$\lambda\geq0$, let
\[
 \widehat a_{u,\lambda}
 =(\widehat{\mathcal B}_S^\top\widehat{\mathcal B}_S+\lambda I)^{-1}
 \widehat{\mathcal B}_S^\top y_S(u).
\]
Let $\widehat\gamma_i$ denote the singular values of
$\widehat{\mathcal B}_S$, and define
\[
 b_\lambda=\frac{\lambda}{\widehat\gamma_r^2+\lambda},
 \qquad
 \rho_\lambda=\max_i
 \frac{\widehat\gamma_i}{\widehat\gamma_i^2+\lambda}.
\]
For $\lambda>0$, or for $\lambda=0$ with full column rank,
\begin{equation}
 \|\widehat a_{u,\lambda}-\bar a_u\|_2
 \leq b_\lambda A+\rho_\lambda(\tau_SA+H_{S,u}+\epsilon_S).
 \label{eq:ridge-coordinate}
\end{equation}
Indeed, the ridge normal equations express the error as a regularization-bias
term plus the ridge inverse applied to
$(\mathcal B_S-\widehat{\mathcal B}_S)\bar a_u+v^{\rm het}_{S,u}+e_S$.
Singular-value decomposition
gives the two displayed coefficients; in particular,
$\rho_\lambda\leq\min\{1/(\underline\sigma_S-\tau_S),1/(2\sqrt\lambda)\}$ for
$\lambda>0$.

\subsection{Unseen-Family Transfer and Decisions}

\begin{proof}[Proof of \cref{thm:end-to-end}]
By \cref{thm:pulse-ident},
$\|\widehat B_\star-B_\star\|_2\leq
\eta_\star/\sigma_r(Q_\star)=:\varepsilon_{B,\star}$. Let
$\zeta_u=(\tau_SA+H_{S,u}+\epsilon_S)/(\underline\sigma_S-\tau_S)$ from
\cref{thm:source-coordinate}. Writing
$r_\star^h(u)=B_\star(\bar a_u+\delta_{\star,u})+d_{\star,u}$ gives
\begin{align*}
 \widehat B_\star\widehat a_u-r_\star^h(u)
 ={}&(\widehat B_\star-B_\star)\bar a_u
 \\
 &+\widehat B_\star(\widehat a_u-\bar a_u)
 -B_\star\delta_{\star,u}-d_{\star,u}.
\end{align*}
Since $\|\widehat B_\star\|_2\leq\|B_\star\|_2+\varepsilon_{B,\star}$, taking norms and
substituting $\|\bar a_u\|_2\leq A$,
$\|\delta_{\star,u}\|_2\leq D_{\star,u}$, $\zeta_u$, and
$\varepsilon_{B,\star}$ gives
Eq. \eqref{eq:e2e-expanded}.
\end{proof}

If $\epsilon_{\star,u}$ is instantiated by the smooth local remainder in
Eq. \eqref{eq:local-remainder}, its curvature term uses
$\|a_{\star,u}\|_2\leq A+D_{\star,u}$ and can therefore scale as
$(A+D_{\star,u})^2$; $H_{S,u}$ accounts only for the linear pooling mismatch.

For completeness, let $\mathcal C_j\subseteq\{+,-\}$ contain the nonempty
source-dead-zone-determinate classes of coordinate $j$, and let $M_{j,s}$ of
$N_{j,s}$ cells in class $s$ fail the sign margin. Over the valid coordinate
set $\mathcal J$,
\begin{equation}
 \operatorname{BA}_{\rm macro}\geq
 1-\frac{1}{|\mathcal J|}\sum_{j\in\mathcal J}
 \frac{1}{|\mathcal C_j|}\sum_{s\in\mathcal C_j}\frac{M_{j,s}}{N_{j,s}}.
 \label{eq:ba-margin}
\end{equation}
When $\rho_J=J(u^\star)-\min_uJ(u)>0$, Eq. \eqref{eq:action-regret} divided by
$\rho_J$ bounds normalized regret. For realized responses, define
$\widetilde J(u)=h v^\top\widetilde r_\star(u)-\operatorname{cost}(u)$,
$\widetilde u^\star\in\arg\max_u\widetilde J(u)$, and
$\widetilde\rho_J=\max_u\widetilde J(u)-\min_u\widetilde J(u)$. Then
\begin{equation}
 \widetilde J(\widetilde u^\star)-\widetilde J(\widehat u)
 \leq h\|v\|_2\big(
 \overline\Delta_{\star,\widetilde u^\star}
 +\overline\Delta_{\star,\widehat u}\big),
 \label{eq:realized-regret}
\end{equation}
and division by positive $\widetilde\rho_J$ gives realized normalized regret.
The realized sign and exact-action conditions replace $\Delta$ by
$\overline\Delta$.

\begin{proof}[Proof of Corollary~\ref{cor:decisions}]
Equation \eqref{eq:e2e} implies
$|\widehat r_{\star,j}(u)-r_{\star,j}^h(u)|\leq\Delta_{\star,u}$ for every
coordinate. A response farther than this radius from zero keeps its sign,
which proves the sign statement.
Every class cell that clears the margin is therefore correct. For coordinate
$j$ and class $s$, recall is at least $1-M_{j,s}/N_{j,s}$. Averaging first over
the nonempty classes of each coordinate and then over valid coordinates proves
Eq. \eqref{eq:ba-margin}.

For action selection, add and subtract predicted utilities and use optimality
of $\widehat u$:
\begin{align*}
 J(u^\star)-J(\widehat u)
 &\leq |J(u^\star)-\widehat J(u^\star)|
 +|\widehat J(\widehat u)-J(\widehat u)|\\
 &\leq h\|v\|_2
 (\Delta_{\star,u^\star}+\Delta_{\star,\widehat u}).
\end{align*}
Division by the positive utility range $\rho_J$ gives the normalized-regret
bound.
For every $u\neq u^\star$, the predicted best-versus-$u$ gap is at least the
true gap minus
$h\|v\|_2(\Delta_{\star,u^\star}+\Delta_{\star,u})$. The stated utility-margin
condition makes all these differences positive, so the predicted maximizer is
$u^\star$.
Applying the same add-and-subtract argument to $\widetilde J$ and the
simultaneous radii $\overline\Delta_{\star,u}$ proves
Eq. \eqref{eq:realized-regret}; division by positive
$\widetilde\rho_J$ gives its normalized form.
For high-probability radii, take a joint event covering every source block,
target pulse, candidate action, and scored coordinate before applying these
deterministic arguments. A union bound over the finite index set constructs
such an event from marginal tail bounds.
\end{proof}

\section{Complete Experimental Protocol}
\label{app:protocol}

The direct pulse and full L-State readouts use the same ridge objective with
$z(s)=p(s)$ and $z(s)=[c(s);p(s)]$, respectively:
\begin{equation}
 \widehat W_{u,\alpha}=\arg\min_W
 \sum_{s\in\mathcal S_{\mathrm{src}}}\|W^\top z(s)-y(s,u)\|_2^2
 +\alpha\|W\|_F^2,
 \label{eq:ridge}
\end{equation}
The structured operator objective used in Section \ref{sec:endpoint_policy} is
\begin{equation}
 \widehat a_u=\arg\min_a\sum_{s\in\mathcal S_{\mathrm{src}}}
 \|R(s)a-y(s,u)\|_2^2+\lambda\|a\|_2^2,
 \qquad \widehat y_{\rm op}(s,u)=R(s)\widehat a_u.
 \label{eq:operator}
\end{equation}

\subsection{State Generation}

Each trajectory starts from the same base model with a newly initialized
zero-output LoRA adapter. Four synthetic generators---sequence reversal,
lexicon classification, date normalization, and template paraphrase---appear
twice in a balanced eight-episode order. States $t00$ through $t08$ are saved,
including adapter and parent hashes. Three fixed trajectory seeds give nine
states per trajectory and 27 per family.

\subsection{Primary Training Episode}

The base model is frozen. LoRA rank is 4 with scaling 8, dropout 0, and query
and value projection targets. Training uses BF16, eager attention, maximum
length 512, micro-batch 1, gradient accumulation 4, AdamW learning rate
$10^{-4}$, betas $(0.9,0.999)$, epsilon $10^{-8}$, no weight decay, and global
gradient clipping at 1.0. Eight optimizer steps consume 32 examples. Every
branch creates new optimizer moments and restores the source adapter afterward;
the restored tensor hash must match exactly.

\subsection{Response Evaluation}

Each capability is the mean negative completion-token NLL after per-token
clipping to $[0,20]$, so larger is better. Before and after evaluations use the
same 32 held-out examples and tokenizer outputs. The primary response is the
paired per-example difference. The main state and readouts use this paired
quantity; generated-answer diagnostics are stored as a separate trace.

\subsection{Leakage Controls}

The data registry binds source IDs, normalized text hashes, processed-file
hashes, and roles. Exact cross-role text intersections are zero; source indices
for public train and evaluation splits are disjoint; synthetic pulse-to-target
13-gram Jaccard is below 0.01. The analysis reads validated completed
manifests, and the label vault opens target responses after the frozen outer
prediction artifact is serialized.

\section{Development Results}
\label{app:development_results}

\begin{table}[H]
\caption{Complete three-family leave-one-family-out development comparison.
The direct pulse and operator methods read the same L-State pulse block.
RMSE is source-standardized; regret is realized normalized regret under six
fixed utilities. The directional blend is evaluated in the development suite.
Target early uses target-specific responses after two update steps.}
\label{tab:development_all}
\centering
\small
\begin{tabular}{lrrrrr}
\toprule
Method & RMSE $\downarrow$ & Cosine $\uparrow$ & Sign BA $\uparrow$ & Regret $\downarrow$ & Top-1 $\uparrow$ \\
\midrule
Source mean & 7.895 & 0.624 & 0.479 & 0.268 & \textbf{0.574} \\
Capability & 9.987 & 0.504 & 0.602 & 0.472 & 0.309 \\
Direct pulse & \textbf{7.772} & 0.620 & 0.531 & 0.299 & 0.500 \\
Scalar statistics & 8.738 & 0.591 & 0.547 & 0.365 & 0.449 \\
Matched state & 7.913 & 0.607 & 0.491 & \textbf{0.232} & 0.566 \\
Operator & 7.777 & 0.672 & \textbf{0.675} & 0.303 & 0.350 \\
Full L-State & 8.293 & \textbf{0.734} & 0.589 & 0.248 & 0.541 \\
Target early (2 steps) & 8.067 & 0.488 & 0.671 & 0.283 & 0.416 \\
Directional blend & 7.876 & 0.628 & 0.643 & 0.292 & 0.422 \\
\bottomrule
\end{tabular}
\end{table}

\begin{table}[H]
\caption{Full L-State readout results by held-out family.}
\label{tab:folds}
\centering
\small
\begin{tabular}{lrrrr}
\toprule
Family & Full L-State & Capability & MSE gain & Sign BA \\
\midrule
Qwen & 0.926 & 2.947 & 0.901 & 0.787 \\
Mistral & 13.866 & 15.117 & 0.159 & 0.610 \\
OLMo & 3.631 & 7.876 & 0.787 & 0.500 \\
\bottomrule
\end{tabular}
\end{table}

\begin{table}[H]
\caption{Family heterogeneity of sign and realized regret.}
\label{tab:family_endpoints}
\centering
\scriptsize
\setlength{\tabcolsep}{3pt}
\begin{tabular}{lrrrrr}
\toprule
Family & Capability BA & Operator BA & Capability R & L-State R & Matched R \\
\midrule
Qwen & 0.829 & 0.743 & 0.497 & 0.375 & 0.352 \\
Mistral & 0.617 & 0.752 & 0.276 & 0.143 & 0.235 \\
OLMo & 0.561 & 0.521 & 0.643 & 0.225 & 0.110 \\
\bottomrule
\end{tabular}
\end{table}

\begin{table}[H]
\caption{Pulse-count ablation for the learned readout.}
\label{tab:pulses}
\centering
\small
\begin{tabular}{rrrrrr}
\toprule
$k$ & RMSE & MSE gain & Cosine & Sign BA & Regret \\
\midrule
1 & 9.108 & 0.168 & 0.670 & 0.604 & 0.300 \\
2 & 9.614 & 0.073 & 0.650 & 0.615 & 0.259 \\
3 & 8.358 & 0.300 & 0.724 & 0.598 & 0.262 \\
4 & 8.293 & 0.311 & 0.734 & 0.589 & 0.248 \\
\bottomrule
\end{tabular}
\end{table}

\subsection{Aligned-Pulse Permutation}

For every outer fold and permutation, we shuffle whole pulse-response matrices
across states within each source family, rebuild the complete 25-dimensional
input, repeat source-trajectory hyperparameter selection, refit, and score the
unchanged target family. The one-sided family results are:
\begin{center}
\small
\begin{tabular}{lrrr}
\toprule
Family & Aligned gain & Null 95th & $p$ \\
\midrule
Qwen & 0.901 & 0.873 & 0.0010 \\
Mistral & 0.159 & 0.135 & 0.0330 \\
OLMo & 0.787 & 0.870 & 0.6014 \\
\bottomrule
\end{tabular}
\end{center}
We report raw per-fold one-sided $p$-values; the preregistered comparison uses
each fold's null 95th percentile.

\subsection{Source-Fit Blend Weights}

Weights below multiply the learned readout; one minus the weight multiplies the
operator readout. Reciprocal source-family fits select the weights while
target-family pulse reliability fixes the sign dead zones and target-action
labels remain sealed.
\begin{center}
\small
\begin{tabular}{lrrrrr}
\toprule
Held out & Math & Code & Science & Reading & General \\
\midrule
Qwen & 0.00 & 0.80 & 0.05 & 0.00 & 0.50 \\
Mistral & 0.05 & 0.00 & 0.30 & 0.75 & 0.05 \\
OLMo & 0.30 & 0.85 & 0.00 & 0.00 & 0.10 \\
\bottomrule
\end{tabular}
\end{center}
This development-suite interpolation raises sign balanced accuracy from 0.589
to 0.643 and reduces RMSE from 8.293 to 7.876. The first GLM seal uses the
registered operator/full-state endpoint assignment summarized in the main paper.

\section{Prospective GLM Details}
\label{app:prospective_details}

\subsection{Freeze Order and Complete Comparison}

The prospective contract first binds the rule learned from the 81 development
states. It then builds 27 GLM states and their four target-independent pulse
responses. While the target branch count remains zero, the pipeline serializes
540 response predictions and 810 utility-conditioned action decisions. Scoring
begins after the label-opening receipt is published and all 108 GLM target
branches finish. Table \ref{tab:prospective_all} gives all five frozen methods.

\begin{table}[H]
\caption{Complete prospective GLM comparison and response slices. Direction
uses the 184 scalar responses selected by source-frozen dead zones; continuous
RMSE retains all 540 components. The direct pulse and operator methods read the
same L-State pulse block.}
\label{tab:prospective_all}
\centering
\small
\begin{tabular}{lrrrrrr}
\toprule
Method & RMSE $\downarrow$ & MSE gain $\uparrow$ & Cosine $\uparrow$ & Sign BA $\uparrow$ & Regret $\downarrow$ & Top-1 $\uparrow$ \\
\midrule
Capability & 1.339 & 0.000 & 0.431 & 0.366 & 0.592 & 0.340 \\
Direct pulse & 0.712 & 0.718 & \textbf{0.552} & 0.571 & 0.451 & 0.389 \\
Matched state & 0.939 & 0.508 & 0.535 & 0.569 & \textbf{0.410} & \textbf{0.432} \\
Operator & \textbf{0.624} & \textbf{0.783} & 0.421 & \textbf{0.754} & 0.513 & 0.352 \\
Full L-State & 0.841 & 0.606 & 0.471 & 0.359 & 0.434 & 0.395 \\
\bottomrule
\end{tabular}
\vspace{8pt}

\begin{minipage}[t]{0.48\textwidth}
\centering
\textbf{Response and direction by target action}\\[2pt]
\resizebox{\linewidth}{!}{%
\begin{tabular}{lrrrrrrr}
\toprule
Action & Cap. RMSE & Direct RMSE & Oper. RMSE & Cap. BA & Direct BA & Oper. BA & Det. \\
\midrule
$u_{\rm math}$ & 1.533 & 0.541 & 0.085 & 0.833 & 0.833 & 0.833 & 45 \\
$u_{\rm code}$ & 1.214 & 0.571 & 0.454 & 0.479 & 0.500 & 0.850 & 48 \\
$u_{\rm science}$ & 0.966 & 0.613 & 0.825 & 1.000 & 1.000 & 0.000 & 28 \\
$u_{\rm reading}$ & 1.554 & 1.015 & 0.815 & 0.167 & 0.438 & 0.641 & 63 \\
\bottomrule
\end{tabular}%
}
\end{minipage}
\hfill
\begin{minipage}[t]{0.48\textwidth}
\centering
\textbf{Response and direction by capability coordinate}\\[2pt]
\resizebox{\linewidth}{!}{%
\begin{tabular}{lrrrrrrr}
\toprule
Coordinate & Cap. RMSE & Direct RMSE & Oper. RMSE & Cap. BA & Direct BA & Oper. BA & Det. \\
\midrule
Math & 1.366 & 0.720 & 0.216 & 0.071 & 0.071 & 0.721 & 19 \\
Code & 1.553 & 0.546 & 0.499 & 1.000 & 1.000 & 1.000 & 52 \\
Science & 1.244 & 0.984 & 1.024 & 0.392 & 0.399 & 0.516 & 86 \\
Reading & 1.686 & 0.852 & 0.772 & 0.000 & 0.815 & 0.778 & 27 \\
General & 0.545 & 0.146 & 0.073 & -- & -- & -- & 0 \\
\bottomrule
\end{tabular}%
}
\end{minipage}
\end{table}

\subsection{Response Slices}

The determinate direction set contains 169 positive and 15 negative responses.
The general coordinate falls inside its source-frozen dead zone for all 108
records, so its direction entry is shown as a dash.

\subsection{Action and Trajectory Slices}

\begin{table}[H]
\caption{Prospective action selection by utility. Cells report regret/top-1.}
\label{tab:prospective_utility}
\centering
\scriptsize
\begin{tabular}{lrrr}
\toprule
Utility & Capability & Full L-State & Matched state \\
\midrule
Math & 0.987/0.000 & \textbf{0.398/0.074} & 0.539/0.037 \\
Code & 0.000/1.000 & 0.000/1.000 & 0.000/1.000 \\
Science & 0.000/1.000 & 0.000/1.000 & 0.000/1.000 \\
Reading & 0.988/0.000 & 0.985/0.000 & \textbf{0.797/0.185} \\
General & 0.749/0.037 & 0.474/0.296 & \textbf{0.375/0.370} \\
Balanced & 0.827/0.000 & \textbf{0.745/0.000} & 0.749/0.000 \\
\bottomrule
\end{tabular}
\end{table}

\begin{table}[H]
\caption{Prospective endpoint estimates by GLM trajectory.}
\label{tab:prospective_trajectories}
\centering
\scriptsize
\resizebox{\columnwidth}{!}{%
\begin{tabular}{lrrrrr}
\toprule
Seed & Operator MSE gain & Operator BA & BA gain & L-State regret & Regret reduction \\
\midrule
20260901 & 0.781 & 0.706 & 0.322 & 0.434 & 0.254 \\
20260902 & 0.785 & 0.735 & 0.372 & 0.409 & 0.299 \\
20260903 & 0.783 & 0.860 & 0.523 & 0.457 & 0.250 \\
\bottomrule
\end{tabular}
}
\end{table}

\begin{table}[H]
\caption{Whole-trajectory bootstrap summaries with 10,000 valid resamples.}
\label{tab:prospective_bootstrap}
\centering
\scriptsize
\begin{tabular}{lrrr}
\toprule
Quantity & Estimate & 2.5\% & 97.5\% \\
\midrule
Operator MSE gain & 0.783 & 0.781 & 0.785 \\
Operator sign BA & 0.754 & 0.705 & 0.860 \\
Operator sign BA gain & 0.388 & 0.322 & 0.523 \\
L-State regret reduction & 0.267 & 0.250 & 0.299 \\
\bottomrule
\end{tabular}
\end{table}

Whole-trajectory resampling gives percentile ranges of
$[78.1\%,78.5\%]$ for operator MSE reduction, $[0.705,0.860]$ for its sign BA,
$[32.2,52.3]$ points for its sign gain, and $[25.0\%,29.9\%]$ for full-state
regret reduction. These intervals summarize variation across the three
prospective GLM trajectories.

\section{Mechanism Diagnosis and Action-Wise Response Selection}
\label{app:revision_diagnostics}

\subsection{Five-Repeat Science Intervention}

The registered repeat phase adds four independent science-action branches to
the original branch for every GLM state. It completes 108 new branches, restores
all state adapters, and leaves the canonical target tree unchanged. Table
\ref{tab:science_repeats} scores the byte-frozen capability, direct, and operator predictions
against the original response and the five-response mean and median.

\begin{table}[H]
\caption{GLM science-action metrics under repeat-based truth summaries. Each
cell is source-standardized RMSE/sign BA.}
\label{tab:science_repeats}
\centering
\small
\begin{tabular}{lrrr}
\toprule
Truth summary & Capability & Direct pulse & Operator \\
\midrule
Original & 0.9664/1.000 & 0.6132/1.000 & 0.8246/0.000 \\
Five-repeat mean & 0.9667/1.000 & 0.6115/1.000 & 0.8245/0.000 \\
Five-repeat median & 0.9643/1.000 & 0.6111/1.000 & 0.8243/0.000 \\
\bottomrule
\end{tabular}
\end{table}

Across the 27 states, the science-coordinate response has mean within-state
standard deviation 0.00446 and unanimous five-repeat sign in 27/27 states. The
source science-action coordinate uses $\alpha=1000$ and has norm 0.0539; its
held-out-family fits have norms 0.0588--1.108 and pairwise cosine
$-0.555$--0.322.

\subsection{Development-Fitted Readout Selection}

The direct readout is selected after a strict held-out-family win in at least
two of the three development folds; the remaining actions retain the operator
readout. Table \ref{tab:source_gate} records the frozen mapping. The direction
selector is retained as a diagnostic; the registered Granite direction endpoint
uses the operator readout.

\begin{table}[H]
\caption{\sourcegate. ``Wins'' counts strict direct-readout wins over the
operator readout among three held-out development families.}
\label{tab:source_gate}
\centering
\small
\begin{tabular}{lrrrr}
\toprule
Action & Response & Wins & Direction & Wins \\
\midrule
Math & Operator & 0/3 & Operator & 0/3 \\
Code & Direct & 3/3 & Operator & 1/3 \\
Science & Direct & 2/3 & Direct & 3/3 \\
Reading & Operator & 1/3 & Direct & 2/3 \\
\bottomrule
\end{tabular}
\end{table}

On Granite, the \sourcegate\ reaches pooled and action-macro RMSE of 0.5537
and 0.5479. Its pooled and action-macro sign BA are 0.6496 and 0.6470. The
source-only direction diagnostic has pooled BA 0.5689 and action-macro BA
0.6558, illustrating the
importance of stating the aggregation rule explicitly.

\section{Five-Family Semantic-Coordinate Audit}
\label{app:sharedness_audit}

\subsection{Operator Recovery and Estimation Contract}

The audit uses 540 previously scored operator family--action--state rows: five
families, four actions, and 27 states per family. The frozen four-action
coordinate matrix is full rank in each of the five family folds. Their
condition numbers are 23.4 for Qwen,
1244.8 for Mistral, 34.7 for OLMo, and 33.6 for both GLM and Granite. Algebraic
operator recovery replays every frozen operator-readout prediction with maximum absolute
error below $10^{-9}$. Re-fitting the three development-family science
coordinates from the recovered operators agrees with the original direct fit
to $3.3\times10^{-15}$, providing a check beyond prediction replay.

The ridge grid spans zero through 10 and is refined between 0.1 and 2. Holding
out a complete trajectory in every family--action cell selects relative ridge
0.7. Each absolute cell penalty equals 0.7 times the mean eigenvalue of that
cell's design Gram. All reported response-scale quantities use the common
three-development-family scale. The family-specific outer-fold scale and every
candidate ridge remain in the sensitivity artifact.

For this retrospective decomposition, the reference set is the five observed
families with $\pi_m=1/5$; write its mean as $\bar a_u^{(5)}$ to distinguish it
from the prospective source-reference coordinate, and suppress the superscript
below. Let $W_u$ be the inverse-squared response scale frozen on the three
development families, and let $Z$ analogously use the frozen reliability dead
zones. For $N$ family--state cells and $d$ response coordinates, define
\begin{equation}
 \begin{aligned}
 D_u^2&=\frac1{Nd}\sum_{m,s}
 \left\|W_u^{1/2}R_m(s)(\beta_m+\gamma_{m,u})\right\|_2^2,\\
 (D_u^{\rm DZ})^2&=\frac1{Nd}\sum_{m,s}
 \left\|Z^{1/2}R_m(s)(\beta_m+\gamma_{m,u})\right\|_2^2,\\
 S_u&=\frac{L_u^2}{L_u^2+D_u^2},\qquad
 L_u^2=\frac1{Nd}\sum_{m,s}\|W_u^{1/2}R_m(s)\bar a_u\|_2^2.
 \end{aligned}
 \label{eq:sharedness-metrics}
\end{equation}
Thus $D_u$ is common-scale pooling error and $S_u$ is the shared-signal
fraction. Exact tests enumerate all $2^{15}$ restricted-residual cluster sign
patterns and use Holm correction across actions. The separate diagnostic
equivalence test is $H_0:D_u^{\rm DZ}\geq1$ against $H_1:D_u^{\rm DZ}<1$; the
dead zones were frozen prospectively, but their use as a sharedness margin was
chosen retrospectively. Tests and predictive comparisons use unpenalized least
squares; ridge 0.7 applies only to coordinate-stability summaries.

\begin{table}[H]
\caption{Direct five-family coordinate stability. Intervals are 10,000
trajectory-block percentile intervals conditional on the five observed
families and selected ridge. ``DZ exceed'' is the fraction of response cells in
which pooling error exceeds the frozen reliability dead zone.}
\label{tab:sharedness_full}
\centering
\scriptsize
\resizebox{\columnwidth}{!}{%
\begin{tabular}{lrrrrrr}
\toprule
Action & Cosine min/mean & Norm ratio & $D_u$ (95\%) & $S_u$ & $D_u^{\rm DZ}$ (95\%) & DZ exceed \\
\midrule
Math & .315/.696 & 3.55 & .342 [.311,.370] & .648 & 2.468 [2.252,2.710] & 24.3\% \\
Code & .448/.653 & 2.43 & .356 [.326,.394] & .800 & 2.620 [2.454,2.815] & 39.7\% \\
Science & $-.770$/.057 & 11.20 & .633 [.605,.672] & .269 & 4.417 [4.324,4.523] & 40.1\% \\
Reading & $-.732$/$-.017$ & 19.43 & 1.038 [.976,1.105] & .404 & 2.837 [2.714,2.998] & 32.3\% \\
\bottomrule
\end{tabular}%
}
\end{table}

\subsection{Difference, Interaction, and Predictive Diagnostics}

The action-specific restricted-residual sign-flip values before/after Holm
adjustment are .0874/.0874 for mathematics, .00165/.00494 for code,
$6.10\times10^{-5}$/$.000244$ for science, and .0239/.0477 for reading.
The joint family-main, total-deviation, and interaction-increment values are
$6.10\times10^{-5}$, $1.83\times10^{-4}$, and $3.05\times10^{-4}$. These are
fixed-15-cluster diagnostics whose validity relies on approximate restricted
residual sign symmetry; they are not tests over a population of model families.

\begin{table}[H]
\caption{Three-fold held-trajectory RMSE under the common response scale. The
family-aware models consume opened labels from every observed family and are
diagnostic, not unseen-family deployment rules.}
\label{tab:sharedness_cv}
\centering
\small
\begin{tabular}{lrrrrr}
\toprule
Coordinate model & All & Math & Code & Science & Reading \\
\midrule
Shared action & .824 & .757 & .887 & .800 & .848 \\
$+$ family main effect & .802 & .794 & .863 & .720 & .823 \\
$+$ family--action interaction & .750 & .756 & .799 & .577 & .840 \\
\bottomrule
\end{tabular}
\end{table}

Compared with the shared model, the full model's descriptive held-trajectory
MSE improvement is 17.3\% overall, 0.3\% for mathematics, 18.8\% for code,
48.0\% for science, and 1.8\% for reading. The three folds have heavily
overlapping training sets, so their resampled ranges are descriptive rather
than formal confidence intervals.

Deleting one family at a time gives mean-cosine ranges .628--.781 for
mathematics, .573--.700 for code, $-.121$--.486 for science, and
$-.136$--.167 for reading. The corresponding $D_u$ ranges are .163--.371,
.285--.427, .500--.752, and .555--1.333. The science and reading effects persist
after removing Mistral, while their magnitude depends on the particular
observed family set. The analysis characterizes the fixed pulse measurements,
chosen ridge, and available target-action training realization across these
five families.

\section{Prospective Granite Details}
\label{app:granite_details}

\subsection{Complete Core Comparison}

The Granite contract binds revision \texttt{4009206d5fc9}, a fixed prompt and
tokenization identity, the 540-row core prediction file, the 810-row action
decision file, and the 216-row selector extension before target labels open.
All 108 target branches complete and the fresh-process replay receipt matches
the canonical core and extension outputs.

\begin{table}[H]
\caption{Complete sealed Granite comparison. RMSE is source-standardized;
direction is coordinate-macro balanced accuracy. The direct pulse and operator
methods read the same L-State pulse block.}
\label{tab:granite_all}
\centering
\scriptsize
\resizebox{\columnwidth}{!}{%
\begin{tabular}{lrrrrrrr}
\toprule
Method & RMSE & MSE gain & Cosine & Sign BA & Regret & Regret red. & Top-1 \\
\midrule
Capability & 1.1716 & 0.0000 & \textbf{0.7382} & 0.6537 & 0.2906 & 0.0000 & 0.5309 \\
Direct pulse & \textbf{0.5439} & \textbf{0.7845} & 0.3155 & 0.5354 & 0.2149 & 0.2604 & 0.6235 \\
Matched state & 0.6916 & 0.6516 & 0.4710 & 0.5351 & \textbf{0.2119} & \textbf{0.2707} & \textbf{0.6481} \\
Operator & 0.6344 & 0.7068 & 0.1725 & 0.6566 & 0.2530 & 0.1296 & 0.5926 \\
Full L-State & 0.6805 & 0.6627 & 0.5079 & \textbf{0.6640} & 0.2166 & 0.2545 & 0.6358 \\
\bottomrule
\end{tabular}%
}
\end{table}

\begin{table}[H]
\caption{Granite operator-readout response and direction by target action.}
\label{tab:granite_actions}
\centering
\small
\begin{tabular}{lrrr}
\toprule
Action & RMSE & Sign BA & Determinate \\
\midrule
Math & 0.5963 & 0.6958 & 79 \\
Code & 0.8878 & 0.8252 & 89 \\
Science & 0.2860 & 0.4256 & 43 \\
Reading & 0.6198 & 0.6420 & 46 \\
\bottomrule
\end{tabular}
\end{table}

\begin{table}[H]
\caption{Granite whole-trajectory bootstrap over the three observed
trajectories (10,000 valid resamples; seed 20260919). Estimate is the aggregate
point estimate; interval columns are percentile bounds.}
\label{tab:granite_bootstrap}
\centering
\small
\begin{tabular}{lrrr}
\toprule
Quantity & Estimate & 2.5\% & 97.5\% \\
\midrule
Operator MSE gain & 0.7068 & 0.6990 & 0.7110 \\
Operator sign BA & 0.6566 & 0.6081 & 0.6584 \\
Operator sign BA gain & 0.0028 & $-0.0615$ & 0.0823 \\
L-State regret reduction & 0.2545 & 0.1565 & 0.4012 \\
\bottomrule
\end{tabular}
\end{table}

Whole-trajectory resampling gives $[69.9\%,71.1\%]$ for the operator MSE reduction,
$[0.608,0.658]$ for its sign BA, $[-6.15,8.23]$ points for its sign gain, and
$[15.7\%,40.1\%]$ for full-state regret reduction. The sealed Granite result
reproduces large response and action gains while replacing the GLM
science-action inversion with non-degenerate direction estimates.

\section{Pulse Identification and Cross-Duration Robustness}
\label{app:duration_details}

\subsection{Pulse Identification and Local-Dynamics Checks}

Theorem \ref{thm:smooth-factor} requires a stable local response measurement,
and Theorems \ref{thm:pulse-ident}--\ref{thm:pulse-opt} connect operator
recovery to the pulse basis. Pulse measurement is stable: the median intraclass correlation is 0.967 and
the median signal-to-noise ratio is 29.2, with 19 of 20 reliability cells at
ICC 0.75 or above. This stability supports endpoint-specific readout selection
within a reliable pulse assay; the development and GLM contrasts show that
direction quality depends on readout geometry.

Pulse-count ablation shows that three to four directions carry the strongest
continuous signal. One, two, three, and four pulses yield RMSE
9.108, 9.614, 8.358, and 8.293, respectively; action regret is 0.300, 0.259,
0.262, and 0.248. Different endpoints favor different pulse counts: sign
balanced accuracy peaks with two pulses at 0.615, while RMSE and regret attain
their best values with four. The near-rank-complete three- and four-pulse
settings deliver the strongest continuous prediction, while the two-pulse
direction peak explains the endpoint-specific readouts.

Finally, controlled pulse mixtures have median relative residual 0.342; changing
exposure from eight steps to four or sixteen gives median relative residual
0.747; and a fixed-step amplitude scan gives 0.757 away from the reference
amplitude. The smaller mixture residual and the duration response quantify the
finite-step geometry induced by the smooth-factorization result.

The cross-duration study reuses nine GLM states and evaluates four-, eight-,
and sixteen-step pulse and target exposures. The operator readout has the best aggregate raw
response RMSE and sign BA among the five core methods at every duration.
The full \lstate{} readout reduces source-standardized RMSE against capability by 45.3\%, 37.3\%,
and 25.1\%, and reduces action regret by 21.8\%, 22.7\%, and 29.3\% at four,
eight, and sixteen steps. Both improvements hold in all nine
duration-by-trajectory slices. Retrospectively fitted operator action coordinates
retain cosine 0.952--0.997 at four steps and 0.973--0.992 at sixteen steps
relative to the eight-step coordinate. This measures within-GLM duration
stability only; it is not evidence for cross-family coordinate sharing. The
registered endpoint aggregates and
coordinate-stability diagnostics appear below; the complete five-method
raw-primary, per-step, and trajectory tables accompany the evidence artifact.

\subsection{Registered Duration Experiment}

The duration experiment evaluates nine GLM states under four-, eight-, and
sixteen-step pulse and target exposures. Raw capability difference is the
primary response; per-step response is a separately reported diagnostic. The
eight-step branches are reused read-only, 144 new branches cover the other two
durations, and every adapter is restored.

\begin{table}[H]
\caption{Registered-endpoint duration results. Response cells report raw
RMSE/sign BA for capability and the operator readout; action cells report
normalized regret/top-1 for capability and the full L-State readout.}
\label{tab:duration_results}
\centering
\small
\begin{tabular}{rrrrr}
\toprule
Steps & Cap. response & Oper. response & Cap. action & L-State action \\
\midrule
4 & .4428/.3564 & \textbf{.1012/.6793} & .5592/.3333 & .4375/.4259 \\
8 & .4529/.3828 & \textbf{.1875/.6623} & .5771/.3333 & .4460/.3889 \\
16 & .4861/.3814 & \textbf{.2878/.7878} & .6063/.3333 & .4285/.3704 \\
\bottomrule
\end{tabular}
\end{table}

At four, eight, and sixteen steps, capability/full-L-State source-standardized RMSE is
1.2801/0.7000, 1.3147/0.8240, and 1.4574/1.0922, respectively, yielding the
reported 45.3\%, 37.3\%, and 25.1\% reductions.

Across all nine duration-by-trajectory slices, the full \lstate{} readout has lower
response RMSE and lower action regret than capability. Duration-specific operator diagnostic
coordinates have cosine 0.952--0.997 at four steps and 0.973--0.992 at sixteen
steps relative to the eight-step coordinate.

\section{Probe Cost Accounting}
\label{app:probe_cost}

\begin{table}[H]
\caption{Configuration-level probe break-even.}
\label{tab:probe_cost}
\centering
\small
\begin{tabular}{lrrr}
\toprule
Resource & Four-pulse probe & One target action & Break-even \\
\midrule
Optimizer steps & 32 & 8 & 4.0 \\
Training examples & 128 & 32 & 4.0 \\
Evaluation examples & 800 & 320 & 2.5 \\
\bottomrule
\end{tabular}
\end{table}

\section{Reproducibility and Artifact Lineage}

The development run binds the frozen data registry, configuration, source tree,
model manifests, state adapters, branch outputs, validated input set, and 18
analysis outputs by SHA-256. Fresh-process replay matches every canonical
development hash.

The first sealed GLM run extends the chain chronologically. It binds the GLM model
snapshot at revision \texttt{645b8482494e}, source tree
\texttt{b975229c6ef7}, frozen rule \texttt{5349b83e483a}, prediction manifest
\texttt{865be9eb5612}, and label-opening receipt \texttt{a5c9bd24048d}. The
prediction manifest records 540 response predictions and 810 action choices
before target training begins. After scoring, the result contains 540 rows and
has SHA-256 \texttt{4cf352bcf3c7}; the scored JSONL has SHA-256
\texttt{271332267cd4}. A fresh process reproduces both canonical outputs and
publishes a PASS receipt bound to analysis manifest \texttt{017469bd7363}.

The revision chain adds registered repeat, duration, selector, and second-family
stages. The action-wise pulse selector has content hash \texttt{fb885c94a096}. The
second sealed Granite run binds source tree \texttt{07738052a7a5}, configuration
\texttt{ef2100d7317f}, core freeze receipt \texttt{8c9758675ebc}, and selector
extension receipt \texttt{54fda8b3a183}; both receipts record zero target
branches. All 108 target branches validate, analysis manifest
\texttt{3c644bad} binds the scored outputs, and fresh replay passes all five
core/extension comparisons. The revision registry records E0--E9 as PASS and
maps each headline number to its replayed evidence file.

The retrospective five-family coordinate audit is a separate derivative stage
that reads those frozen scored rows without retraining. Its analysis manifest
has SHA-256 \texttt{e05a19e13c74}; the copied deterministic analysis script has
SHA-256 \texttt{13ba7e5f46db}. The manifest binds every reported table, the
10,000-resample seed and count, the exhaustive sign-flip contract, and the
operator-replay checks.

\end{document}